%% file: main_icml.tex
\documentclass{article}
\usepackage[accepted]{icml2024}

\input{math_commands.tex}

\usepackage{amsmath}
\usepackage{stfloats} 
\usepackage{hyperref}
\usepackage{url}
\usepackage{booktabs}
\usepackage{graphicx}
\usepackage{multirow}
\usepackage{adjustbox}
\usepackage{float}
\usepackage{caption}[small]
\usepackage{subcaption}
\usepackage[colorinlistoftodos]{todonotes}
\usepackage{lscape}
\usepackage{rotating}
\usepackage{lineno}
\usepackage{enumitem}
\usepackage{wrapfig}

\usepackage{xcolor}
\usepackage{paralist}
\usepackage{nicefrac}       %
\usepackage{algorithm}
\usepackage{algorithmic}
\usepackage{cleveref}
\usetikzlibrary{positioning}

\theoremstyle{plain}
\newtheorem{theorem}{Theorem}[section]
\newtheorem{proposition}[theorem]{Proposition}

\theoremstyle{definition}

\theoremstyle{remark}

\begin{document}
\twocolumn[
\icmltitle{On The Fairness Impacts of Hardware Selection in Machine Learning}
\icmlsetsymbol{equal}{*}
\begin{icmlauthorlist}
\icmlauthor{Sree Harsha~Nelaturu}{equal,c4aic,saar}
\icmlauthor{Nishaanth K~Ravichandran}{equal,c4aic}
\icmlauthor{Cuong Tran}{dyan,uva}
\icmlauthor{Sara Hooker}{c4ai}
\icmlauthor{Ferdinando Fioretto}{uva}
\end{icmlauthorlist}

\icmlaffiliation{c4aic}{Cohere For AI Community}
\icmlaffiliation{saar}{Saarland University}
\icmlaffiliation{uva}{University of Virginia}
\icmlaffiliation{c4ai}{Cohere For AI}
\icmlaffiliation{dyan}{Dyania Health}

\icmlcorrespondingauthor{Sara Hooker}{sarahooker@cohere.com}
\icmlcorrespondingauthor{Ferdinando Fioretto}{fioretto@virginia.edu}
\author{
    name={Sree Harsha Nelaturu},
    affiliation={Cohere For AI Community, Saarland University},
    email={nelaturu.harsha@gmail.com}
}
\author{
    name={Nishaanth K.~Ravichandran},
    affiliation={Cohere For AI Community},
    email={nishaanthkanna@gmail.com}
}
\author{
    name={Cuong Tran},
    affiliation={University of Virginia},
    email={cuongtran@virginia.edu}
}
\author{
    name={Sara Hooker},
    affiliation={Cohere For AI},
    email={sarahooker@cohere.com}
}
\author{
    name={Ferdinando Fioretto},
    affiliation={University of Virginia},
    email={fioretto@virginia.edu}
}
\icmlkeywords{Machine Learning, ICML}
\vskip 0.3in
]
\printAffiliationsAndNotice{\icmlEqualContribution}

\sloppy\allowdisplaybreaks

\begin{abstract}
In the machine learning ecosystem hardware selection is often regarded as a mere utility, overshadowed by the spotlight on algorithms and data. This is especially relevant in contexts like machine learning as-a-service platforms, where users often lack control over the hardware used for model training and deployment. This paper investigates the influence of hardware on the delicate balance between model performance and fairness. We demonstrate that hardware choices can exacerbate existing disparities, and attribute these discrepancies to variations in gradient flows and loss surfaces across different demographic groups. 
Through both theoretical and empirical analysis, %
the paper not only identifies the underlying factors but also proposes an effective strategy for mitigating hardware-induced performance imbalances.
\end{abstract}

\section{Introduction}

The leap in capabilities of modern machine learning (ML) models has been powered primarily by the availability of large-scale datasets, gains in available compute, and the development of algorithms that can effectively use these resources \citep{radford2019language,brown2020language}. 
As ML-based systems become integral to %
decision-making processes that bear considerable social and economic consequences, questions about their ethical application inevitably surface. 
While an active area of research has been devoted to understanding algorithmic choices and their implications on fairness \citep{hooker2020characterising,quan2023learning,caton2020fairness}  and robustness \cite{carlini2017towards,waqas2022exploring} in neural networks,
there has been limited work to date concerning the influence of hardware tooling on these critical aspects of model performance \citep{Hooker2020TheHL,zhuang2022randomness,JeanPaul2019IssuesIT}. 

This inquiry is especially pertinent as the ML hardware landscape undergoes substantial diversification, from successive generations of GPUs %
to custom deep-learning accelerators like TPUs \citep{TPU}. %

This is significant, as ML models are frequently trained on ML services where there is limited choice in hardware, often geared towards increasingly specialized AI hardware, encouraging the use of a narrow range of ML frameworks \citep{mince2023grand}. More importantly, recent studies have indicated that models trained on different hardware can exhibit varying levels of accuracy %
\citep{zhuang2022randomness}. 
A potential explanation is that hardware-induced nuances, such as precision discrepancies and threading behaviours, may lead iterative optimizers to different local minima during training \citep{Hooker2020TheHL}.

This paper further shows that these hardware-induced variations can disproportionately affect different groups, leading to a ``rich get richer, poor get poorer'' dynamic. This effect is depicted in Figure \ref{fig:motivation}, which shows the variable impact of hardware changes on both a facial recognition task accuracy (left) and on an image classification task (right) across demographic groups or classes. Note that the only variable factor in this study is the underlying GPU adopted for training, while all other sources of software-related randomness are controlled.
Remarkably, while the accuracy rates for majority groups (illustrated with lighter colors) remain relatively stable across different hardware configurations, the rates for minority groups (darker colors) exhibit considerable variability (left). This disparity also arises in balanced datasets (right).

\begin{figure*}
    \centering
        \includegraphics[width=0.475\textwidth]{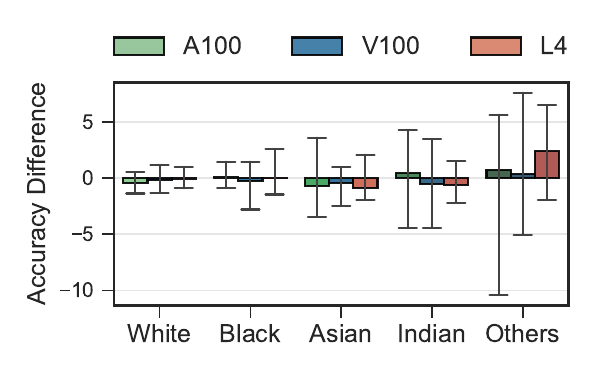}
        \includegraphics[width=0.475\textwidth]{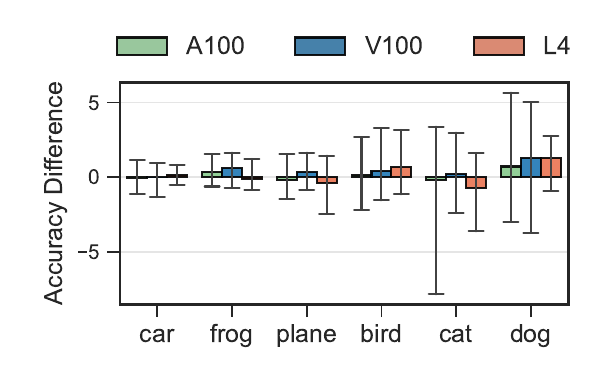}
    
        \caption{A model (ResNet34) with the same parameters (random seeds, epochs, batch-size) on different hardware can have vastly different performance results, especially for minority groups (dark colors). The reference hardware is T4. {\bf Left}: UTK-Face, {\bf Right}: CIFAR-10.}
        \label{fig:motivation}
\end{figure*}

Building on these observations, this work leverages a theoretical framework aimed at quantifying hardware-induced performance disparities and reveals, through an extensive empirical validations, that hardware choices systematically alter not just accuracy but also fairness. 
Our findings suggest that two key mechanisms contribute to these disparities: {\bf (1)} variations in gradient flows across groups, and {\bf (2)} differences in local loss surfaces. Informally, the former affects local optimality for groups, while the latter pertains to model separability. We analyze these contributing factors in detail, providing both theoretical and extensive empirical experiments across various hardware configurations, networks, and datasets. Additionally, by recognizing these factors, we propose a simple yet effective technique that can be used to mitigate the disparate impacts caused by hardware tooling. The proposed method relies on an alteration to the training procedure to augment the training loss with the factors identified as responsible for unfairness to arise. 

Our study stands out for its breadth, conducting experiments that cover a range of hardware architectures, datasets, and model types and the reported results highlight the critical influence of hardware on both performance and ethical dimensions of machine learning models.

\section{Related Work}
The intersection of hardware selection and fairness in ML is an emerging area of research that has received limited attention. For example, the stochastic effects introduced by software dependencies, such as compilers and deep learning libraries, have been recently shown to impact model performance \cite{Hong2013,Pham2020}. However, these studies have evaluated these effects within the constraints of specific setups, leaving a gap in understanding how hardware selection affect fairness in ML.

Research on ML fairness has predominantly focused on algorithmic aspects. The interplay between fairness and efficiency has been examined through the lens of model compression techniques like pruning and quantization %
\cite{xu2022can, ahia2021low, TFKN:neurips22}. 
Another possibly related line of work is that which looks at the relationship between fairness and privacy in ML systems \cite{bagdasaryan2019differential,cummings2019compatibility,TMF:neurips21,TF:ijcai23,TFHY:ijcai21,ZFH:ijcai22,DRF:icml24}. 
In particular, Differential Privacy \cite{dwork:06}, an algorithmic property often employed to protect sensitive data in data analytics tasks, has been shown to conflict with fairness objectives. \citep{FHZ:ijcai22} surveys the recent progress in this area, exploring this tension, and suggesting that achieving both privacy and fairness is not solely a data-driven issue but may also require careful algorithmic design. 

Finally, the influence of randomness introduced through algorithmic choices, including the impact of random seed, initialization, and data handling, has been a focal point of research.
\citet{summers2021nondeterminism} benchmark the separate impact of choices of initialization, data shuffling and augmentation. \citet{Ko2023FAIREnsembleWF} evaluate how ensembling can mitigate unfair outcomes. 
Another body of scholarship has focused on sensitivity to non-stochastic factors including choice of activation function and depth of model \citep{Snapp2021,shamir2020smooth}, hyper-parameter choices \citep{lucic2018gans,Henderson2017,kadlec2017,MLSYS2021_cfecdb27}, the use of data parallelism \citep{shallue2019measuring} and test set construction \citep{sogaard-etal-2021-need,lazaridou2021pitfalls,Melis2018OnTS}. 
While these factors are critical in training phases, they do not study how hardware selection may influence the model outcomes. 

Our analysis is inspired by the work of \cite{TMF:neurips21}, who study the disparate impact of learning under differential privacy. Our aim is to offer insights into how hardware choices can impact the balance between model performance and fairness. 

\begin{figure*}[!t]
    \centering
     \includegraphics[width=\textwidth]{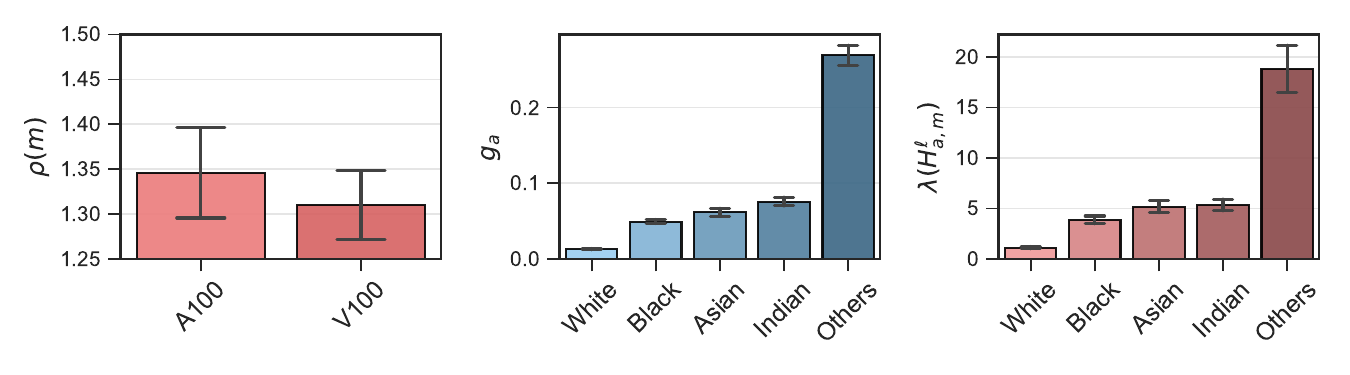}
     \caption{Illustration of the three main components in Theorem \ref{thm:taylor}. {\bf Left}: Difference in model parameter $\rho(m) = \max_m' \| \btheta^{*}_{m} -\btheta^{*}_{m'}\|_2$ when $m=T4$. {\bf Middle}: Gradient flows $\|\bm{g}^\ell_a\|$ on T4 for five races. {\bf Right}: Hessian max eigenvalues $\lambda(\bm{H}^{\ell}_{a,m})$ on T4 for five races. }
        \label{fig:2}
\end{figure*}

\section{Preliminaries}
\label{sec:preliminaries}
We consider a dataset $D$ consisting of $n$ datapoints 
$(\bm{x}_i, a_i, y_i)$, with $i \in [n]$, drawn i.i.d.~from an unknown
distribution $\Pi$. Therein, $\bm{x}_i \in \cX$ is a feature vector,
$a_i \in \cA$ with $\cA = [g]$ (for some finite $g$) is a demographic group 
attribute, and $y_i \in \cY$ is a class label. For example, in a face recognition task, the training example feature $\bm{x}_i$ may 
describe a headshot of an individual, the protected attribute $a_i$
the individual's gender or ethnicity, and $y_i$ the identity of the individual. 
The goal is to learn a predictor $f_\btheta : \cX \to \cY$, where $\btheta$
is a $k$-dimensional real-valued vector of parameters that minimizes 
the empirical risk function:
\begin{align}
\label{eq:erm}
\optimal{\btheta} \,= \argmin_\btheta J(\btheta; D) = 
\frac{1}{n} \sum_{i=1}^n \ell(f_\btheta(\bm{x}_i), y_i),
\end{align}
where $\ell: \cY \times \cY \to \RR_+$ is a non-negative \emph{loss function} that measures the model quality. 

\emph{In this work, our focus is on analyzing the impact of different hardware, used when 
optimizing the above expression, in relation to the model fairness (as defined next).} 
The paper uses $\optimal{\btheta}_m$ to denote the parameters of a model training on hardware
$m \in \cM$, the set of all possible hardware.

\smallskip\noindent\textbf{Fairness.} The fairness analysis focuses on the notion of \emph{hardware sensitivity}, defined 
as the difference among the risk functions of some protected group $a$ of models trained on a different hardware $m'$ from a reference hardware $m$: 
\begin{equation}
\label{eq:2}
    \Delta(a,m) = \max_{m' \in \cM} | J(\optimal{\btheta}_{m}; D_a) -   J(\optimal{\btheta}_{m'}; D_a) |.
\end{equation}
Therein, $D_a$ denotes the subset of $D$ containing samples ($\bm{x}_i, a_i, y_i$) whose group membership $a_i = a$. 
Intuitively, the hardware sensitivity represents the change in loss (and thus, in accuracy) that a given group experiences as a result of hardware tooling.
Fairness is measured in terms of the maximal \emph{hardware loss difference}, 
also referred to as \emph{fairness violation} across all groups:
\begin{equation}
\label{eq:3}
    \xi(D, m) = \max_{a, a' \in \cA} |\Delta(a, m) - \Delta(a', m)|,
\end{equation}
defining the largest hardware sensitivity across all protected groups. 
A fair training method would aim at minimizing the hardware sensitivity across different hardware. 

The goal of the paper is to shed light on why fairness issues arise when the only difference in training aspects of a model is the hardware in which it was trained.

\section{Fairness analysis in tooling}

To gain insights into how tooling may introduce unfairness, we start by providing a useful bound for the hardware sensitivity of a given group. Its goal is to isolate key aspects of tooling that are responsible for the observed unfairness. The following assumes the loss function $\ell(\cdot)$ to be at least twice differentiable, which is the case for common ML loss functions, such as mean squared error or cross-entropy loss. We report proofs of all theorems in the Appendix.

\begin{figure*}
    \centering
     \includegraphics[width=\textwidth]{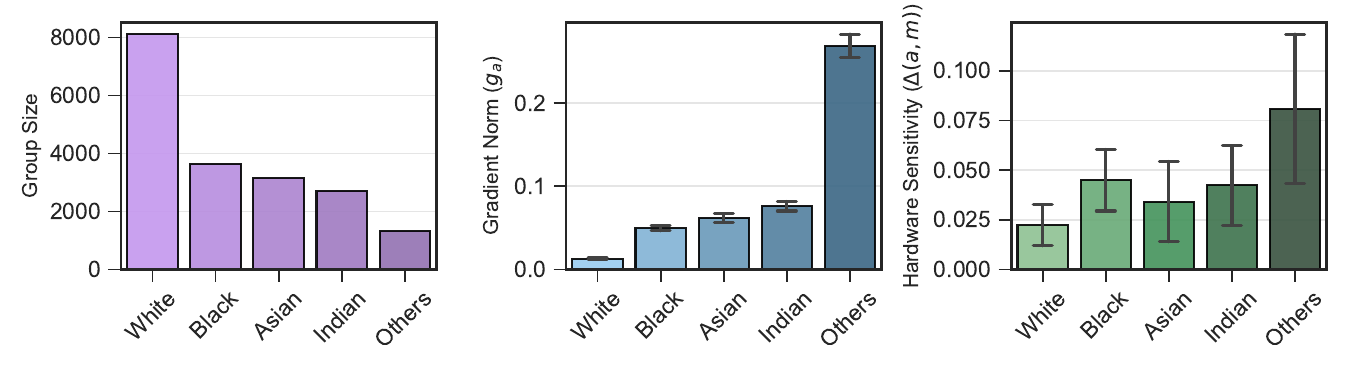}
    \caption{Illustration of Impact of group size on Gradient Norm Imbalance as shown in Theorem \ref{thm:3}. {\bf Left}: Group size used in training for five races. {\bf Middle}: Gradient norms $\bm{g}_a$ averaged across three devices for five races and 10 seeds each. {\bf Right}: Hardware Sensitivity; Notice higher sensitivity as the group size decreases.}
\label{fig:group_size}
\end{figure*}

\begin{theorem}
\label{thm:taylor} 
Given reference hardware $m$, the \textbf{hardware sensitivity} $\Delta(a,m)$ of group $a \in \cA$ is upper bounded by:
\begin{align}
  \Delta(a, m)  \leq &\left\| \bm{g}_{a,m}^{\ell} \right\|_2 \times  \rho(m) 
   + \frac{1}{2} \lambda\left( \bm{H}_{a,m}^{\ell} \right) \times \rho(m)^2 \notag\\
  & + \cO\left( \rho(m)^3 \right),
  \label{eq:thm1}
\end{align}
where $\rho(m) = \max_{m \in \cM} \| \btheta^*_m -\btheta^*_{m'}\|_2$, $\bm{g}_{a,m}^\ell = \nabla_{\btheta} J( \optimal{\btheta}_m; D_{a})$, and $\bm{H}_{a,m}^{\ell} = \nabla^2_{\btheta} J(\optimal{\btheta}_m; D_{a})$, with $\lambda(\Sigma)$ as the maximum eigenvalue of matrix $\Sigma$.
\end{theorem}
The upper bound is derived using a second-order Taylor expansion and Rayleigh quotient properties and is inspired by the analysis of \citet{TMF:neurips21,TFHY:ijcai21}.

Firstly, empirically, we find that this upper bound closely approximates the hardware sensitivity in practice, 
while also observing that the contribution of the third-order term in Equation (\ref{eq:thm1}) is negligible. This empirical validation is 
consistent with observations in existing literature \cite{vadera2022methods,gu2021sde}.

Next, we note that the constant factor $\rho(m)$ is non-zero, as evidenced by Figure \ref{fig:2} (left). These two observations emphasize the presence of two key group-dependent terms in Equation (\ref{eq:thm1}) that modulate hardware sensitivity and form the crux of our fairness analysis. Specifically, they are 
{\bf (1)} the norms of the gradients $\bm{g}_{a,m}^\ell$ (also called gradient flows) and {\bf (2)} the maximum eigenvalue of the Hessian matrix $\bm{H}_{a,m}^\ell$ for a given group \( a \) and reference hardware \(m\). Informally, the first term relates to the local optimality within each group, whereas the second term is indicative of the model's capacity to distinguish between different groups' data. Figure \ref{fig:2} provides an illustration of the disparity of these components across protected groups. 
We will subsequently demonstrate that these components serve as the primary sources of unfairness attributed to tooling.

The next sections analyze the effect of gradient flows and the Hessian to unfairness in the models trained on different hardware. 
This understanding, besides clarifying the roles of these components onto (un)fairness, will help us designing an effective mitigation technique, introduced in Section~\ref{sec:mitigation}.

\subsection{Group Gradient Flows}

Theorem \ref{thm:taylor} illustrates that a key determinant of unfairness in hardware selection lies in the differences in gradient flows across groups. It points out that larger gradient flows within a group are associated with increased hardware sensitivity for that group.

The size of the training group, in particular, plays a significant role in determining the associated gradient flows. Theorem~\ref{thm:2} explores this phenomenon in binary group settings, illustrating how differences in group sizes can lead to distinct gradient norms. Subsequently, Theorem \ref{thm:3} broadens the scope of this analysis to multi-group contexts, under mild assumptions. 

\begin{theorem}
\label{thm:2}
Consider a local minima $\btheta^*_m$ of Equation (\ref{eq:erm}) on a reference hardware $m \in \cM$ and let the set of protected groups be $\cA = \{a, b\}$. If  $|D_a| > |D_b|$ then $\| \bm{g}_a \| < \| \bm{g}_b\|$.
\end{theorem}

The proof is derived by leveraging the conditions for a local minimum and the proportional contributions of each group to the total gradient. This result explains why \emph{smaller groups yield larger gradient norms, which consequently amplify sensitivity to stochasticity introduced by hardware}, as observed in our experimental results. We next  generalize these insights to arbitrary group sets \( \cA \).
\begin{theorem}
\label{thm:3} 
Consider a particular hardware $m \in \cM$, suppose for any group $a, a' \in \cA$ the angle between two gradient vectors $\overrightarrow{\boldsymbol{g}^{\ell}_{a,m}}; \overrightarrow{\boldsymbol{g}^{\ell}_{a',m}}$ is smaller than $\frac{\pi}{2}$. Then if we $ \underline{a} = \min_{a \in \cA} |D_a|$, i.e., the group with least number of training samples then the following holds:
\(
\| \bm{g}^{\ell}_{\underline{a},m}\| = \max_{a \in \cA} \| \bm{g}^{\ell}_{a,m}\|.
\)
\end{theorem}

Theorem \ref{thm:3} suggests that the group with the smallest number of training samples will exhibit the largest gradient norm upon convergence. The underlying assumption—that the angle between any pair of gradient vectors is less than \( \frac{\pi}{2} \)—essentially posits that the learning tasks across different groups are not highly dissimilar, which is often observed in practice \cite{guangyuan2022recon}.

This influence of group size on gradient norm is exemplified in Figure \ref{fig:group_size}. Within the UTK-Face dataset, the \textsl{White} category has the highest number of training samples, while \textsl{Others} has the fewest (left plot). Consequently, the majority group (\textsl{White}) exhibits the smallest gradient norm, while the group \textsl{Others} shows the largest (Figure \ref{fig:group_size} middle). As corroborated by Theorem \ref{thm:taylor}, which establishes the link between gradient norm and hardware sensitivity (unfairness), the majority group manifests the least sensitivity, whereas the minority one has the highest. This is illustrated in the right subplot of Figure \ref{fig:group_size}. Additional empirical evidence supporting the impact of group sizes on gradient norms is provided in Section \ref{sec:experiments}.

\subsection{Group Loss Landscape}
\label{sec:hessian}

While the previous section reviewed the influence of gradient flows on the unfairnes observed in tooling, Theorem \ref{thm:taylor} introduces another critical variable in determining hardware sensitivity: the \emph{eigenvalues of the group Hessians} and suggests that groups with larger eigenvalues are more susceptible to higher hardware sensitivity. Intuitively, the eigenvalues associated to the group Hessian serves as an indicator for the \emph{flatness} of the loss landscape around the optimal solution \cite{li2018visualizing}, as well as for the model's generalization capability \cite{kaur2023maximum}. 

\begin{figure}[t]
	\centering
    \includegraphics[width=\linewidth, height=110pt]{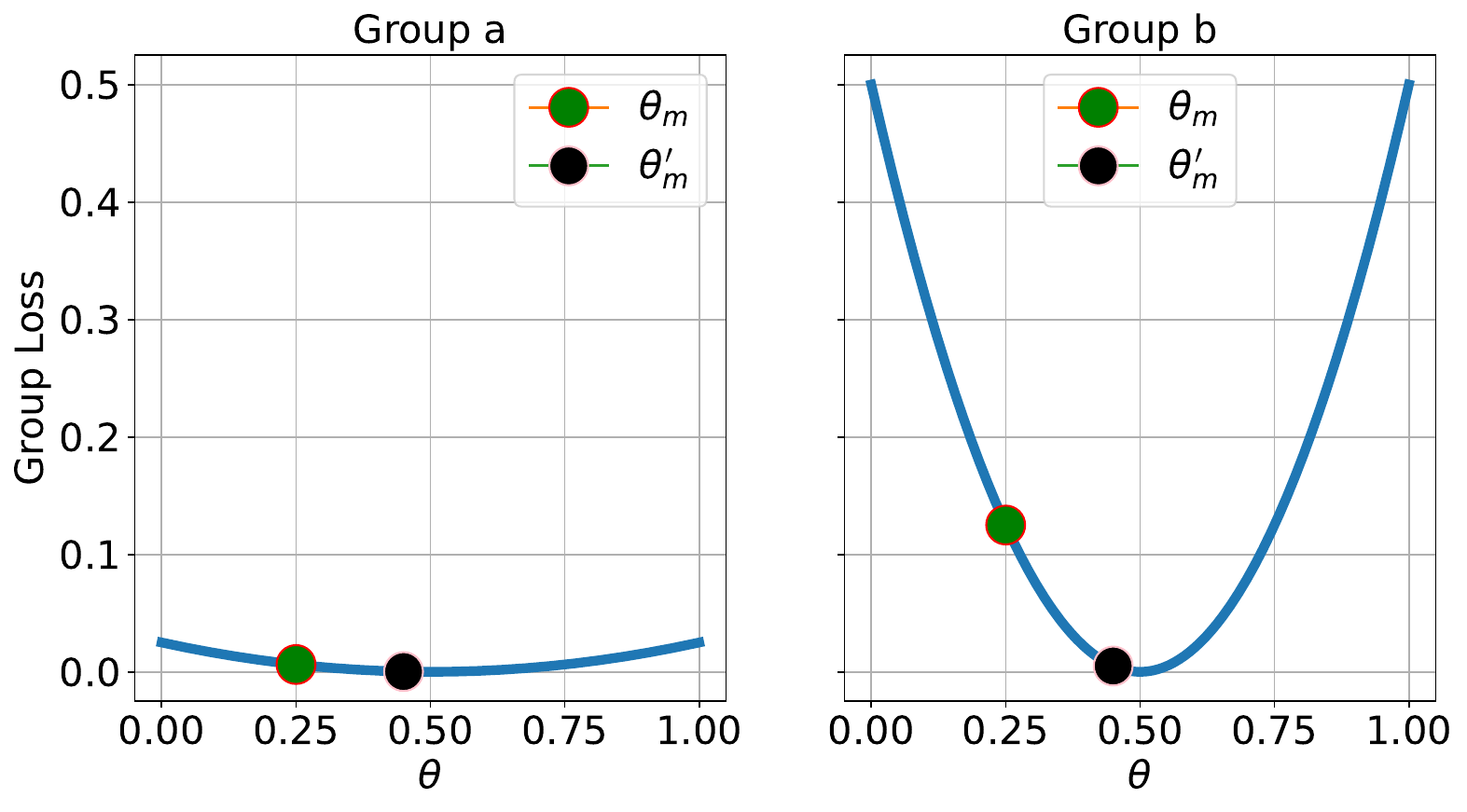}
	    \caption{\small Illustration on the impact of group Hessians. Group 'a' has a smaller Hessian compared to Group 'b', resulting in lower sensitivity of the loss function for Group 'a'.}
        \label{fig:illustration_hessian_norm}
\end{figure}
For illustrative purposes, Figure \ref{fig:illustration_hessian_norm} represents how differences in Hessians maximum eigenvalues impact hardware sensitivity. In this example, group \( a \) has a flatter loss landscape around the stationary point compared to group \( b \) due to its smaller Hessian norm. As a result, variations in model parameters \( \btheta^*_m \) and \( \btheta^*_{m'} \) across hardware $m$ and $m'$ lead to a much smaller change in the loss function for group \( a \) than for group \( b \). This difference underscores the direct link between Hessian norm and the degree of hardware sensitivity experienced by each group.

The next result sheds light on the underlying reasons for the observed disparities in group-specific Hessians. Theorem \ref{thm:4} establishes a relationship between the maximum eigenvalues of the group Hessian and the average distance of samples within that group to the decision boundary.

\begin{theorem}
\label{thm:4}
Consider a model $f_{\btheta^*_m}$ trained using binary cross entropy on reference hardware $m$. Then, $\forall a \in \cA$, the maximum eigenvalue of the group Hessian $\lambda(\bm{H}_a^{\ell})$ is bounded by:
\begin{align*}
    \lambda(\bm{H}_a^{\ell}) &\leq \frac{1}{|D_a|} \sum_{(\bm{x}, y)\in D_a} \delta_{\bm{x}} \times \left\| \nabla_{\btheta} f_{\btheta^*_m}(\bm {x}) \right\|^2 \\
    &\quad + \left| f_{\btheta^*_m}(\bm{x}) - y \right| \times \lambda\left( \nabla^2_{\btheta} f_{\btheta^*_m}(\bm{x}) \right),
\end{align*}
where $\delta_{\bm{x}} = \left( {f}_{\btheta^*_m}(\bm{x}) \right) \left( 1 - {f}_{\btheta^*_m}(\bm{x}) \right)$ is the distance to decision boundary and ${f}_{\btheta}(\bm{x}) \in [0, 1] $ is the output obtained after the last (Sigmoid) layer. 

\end{theorem}

This theorem relies on derivations of the Hessian associated with model loss function and Weyl inequality provided in \Cref{thm:2}.
In other words, Theorem \ref{thm:4} shows that the maximum eigenvalue of the group-specific Hessian is directly linked to how close the samples from that group are to the decision boundary, as measured by the term \( {f}_{\btheta^*_m}(\bm{x})  ( 1 - {f}_{\btheta^*_m}(\bm{x}))\). Intuitively, this term is at its maximum when the classifier is most uncertain about its prediction, meaning when \( {f}_{\btheta^*_m}(\bm{x}) \) is close to 0.5. Conversely, it reaches a minimum when the classifier is most certain, that is, when \( {f}_{\btheta^*_m}(\bm{x}) \) approaches either 0 or 1. This relationship is further elaborated in the subsequent proposition.

\begin{proposition}
\label{prop:dist_bndry}
 Consider classifier $f_{\btheta*_m}(\bm{x})$ trained on hardware $m$. For a sample $\bm{x} \in D $, the term ${f}_{\btheta^*_m}(\bm{x}) (1 - {f}_{\btheta^*_m}(\bm{x}) )$ is maximized when ${f}_{\btheta^*_m}(\bm{x})  = 0.5$ and minimized when ${f}_{\btheta^*_m}(\bm{x}) \in \{0,1\} $.
\end{proposition}

An empirical illustration, shown in Figure  \ref{fig:7}, highlights the relationship  between group Hessian eigenvalues and proximity to the decision boundary. Notice how samples from  the \textsl{Others} group are closer to the decision boundary, indicating that they are less \emph{separable} than those in other groups. As a result this group reports the largest eigenvalue of group Hessians. Similar observation on other datasets are discussed in the following section.

\begin{figure}[!tb]
    \centering
         \includegraphics[width=\linewidth]{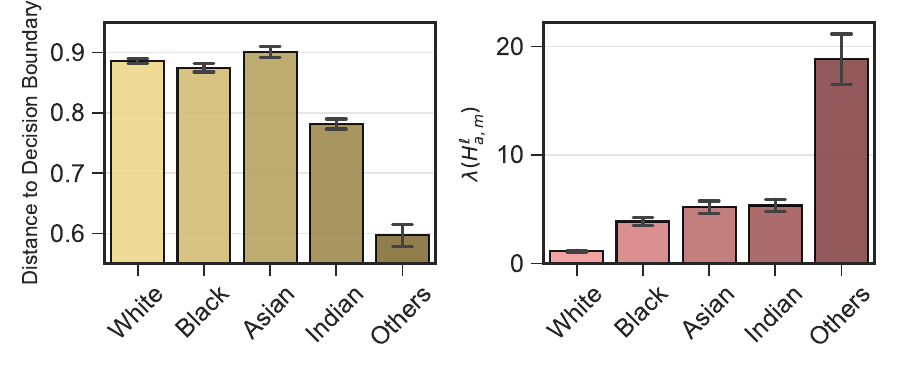}
        \caption{The relationship between Hessian norm and distance to the decision boundary.}
        \label{fig:7}
\end{figure}

\smallskip

Having discussed the main reasons justifying unfairness in hardware selection, the next sections describes the empirical validation and discuss a possible mitigation solution.

\section{Experimental Setup}
We first review the experimental setup.

\textbf{Hardware selection.}~ 
The experiments use a variety of GPUs: Tesla T4 \citep{turing}, Tesla V100 \citep{volta}, Ampere A100 \cite{amperewhitepaper}, and Ada L4 GPU \citep{ada}. These GPUs differ in CUDA core count, total threads, and streaming multiprocessors (refer to Table \ref{tab:gpu_comparison} in the Appendix for details). Given the lack of detailed public information on the internal workings of these devices, we relied on generational differences as key indicators for our hardware selection, also considering availability and suitability while selecting the devices. In addition, we strictly control software related randomness and consider the fact that these GPUs show stochastic behavior in outputs due to varied floating-point processing as detailed in \cite{dab} and parallelization. For example, the T4 and L4 GPUs, being designed for inference, have lower memory requirements and distinct parallelization designs. %

\paragraph{Controlling other sources of stochasticity.} 
Our primary aim is to isolate the impact of hardware on model fairness and performance. This requires to control all other factors as much as possible. We ensure determinism by fixing the random seed for all python libraries, and ensuring consistent data-loading order and augmentations with FFCV-SSL \cite{FFCV_SSL}. This approach ensures that the same stochastic elements are present during training and inference. We also used the same library versions, including cudNN \cite{cudnn} with full-precision (FP32) except in the case of CelebA dataset where we used mixed-precision training. We compute hardware sensitivity by calculating differences for identical random seeds across various devices, and ensuring that all other sources of randomness are fixed, then average these over multiple seeds and report the mean and standard deviation. We also report mean and standard deviation of other metrics averaged across multiple random seeds. \emph{This approach allows us to confidently attribute any observed variations in sensitivity or stochasticity specifically to the unique characteristics of the hardware platform.}

\paragraph{Datasets and architectures.} 
Our experiments were conducted using three key datasets: CIFAR-10 \cite{cifar10}, CelebA 
\cite{celeba}, and UTKFace \cite{utkface}. CIFAR-10 is balanced, while UTKFace and CelebA are naturally imbalanced. To study class imbalance in CIFAR-10, we create a variant where class 8 (\textsl{Ship}) is reduced to $20\%$  of its original size. For CelebA, we redefined the task into four classes based on \textsl{male} and \textsl{blond hair} attributes, creating an imbalanced multi-class dataset. In UTKFace, ethnicity labels are used for training. Additional details are provided in \Cref{app:datasets}.

The evaluations is performed across multiple hardware setups, hyper-parameters, datasets, and four architectures of increasing complexity: SmallCNN, ResNet18, ResNet34, and ResNet50 \cite{resnet}. 
More information regarding these architectures is provided in \Cref{app:architectures} %
For all experiments We used SGD+momentum $0.99$, with weight decay of $5e-4$, a three-phase one-cycle LR \cite{smith2015cyclical} scheduler with a starting learning rate of $0.1$. The batch size for CIFAR10, UTKFace and CelebA is set to $512$, $128$ ($32$ for ResNet50) and $200$, respectively. 
Models trained on CIFAR10 and CelebA were trained for $15$ epochs, and those trained on UTKFace for $20$ epochs. Further analysis on model convergence and impact on hardware sensitivity are provided in  \Cref{app:convergence}.

\begin{figure}[t]
     \centering     
         \includegraphics[width=0.9\linewidth, trim={0 10pt 0 0}, clip]{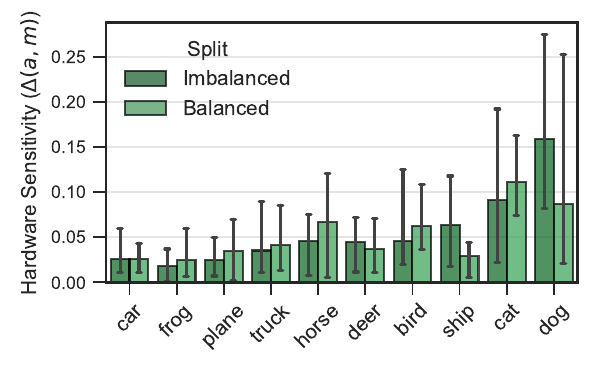}\\
         \includegraphics[width=0.9\linewidth, trim={0 10pt 0 4pt}, clip]{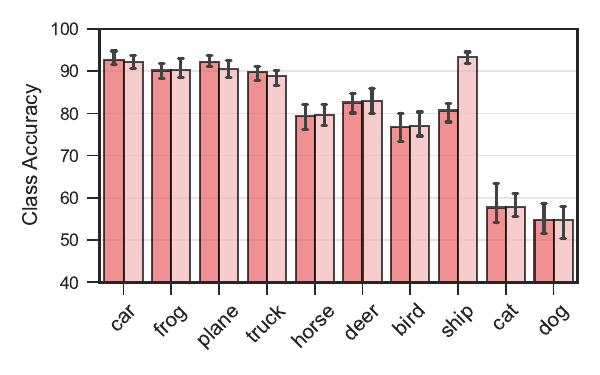}
         \caption{\textbf{Top:} Hardware sensitivity for CIFAR10 (ResNet34) Balanced (lighter) and Imbalanced (darker). %
         \textbf{Bottom:} Class-wise accuracy. %
          High Fairness violation are noted for the Imbalanced CIFAR10.}
         \label{fig:hardware_sensitivity}
\end{figure}

\begin{figure*}[ht!]
     \centering
     \begin{subfigure}[b]{\textwidth}
         \centering
         \includegraphics[width=\textwidth]{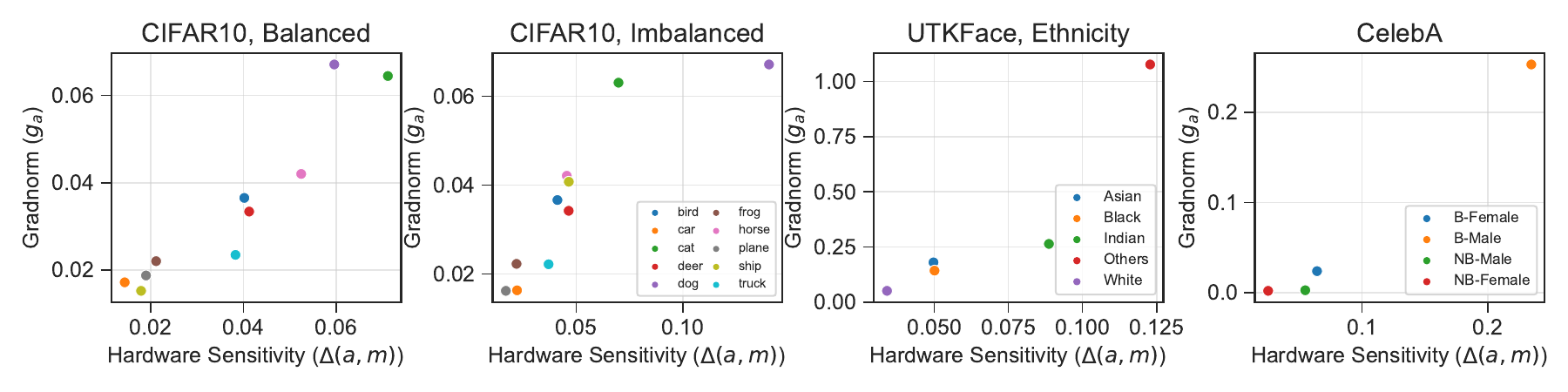}
     \end{subfigure}
        \caption{
Correlation plot between Hardware sensitivity and gradient flows. \textbf{First:} Even with a perfectly balanced dataset, classes with higher gradient norm tend to have higher sensitivity in the change of hardware. \textbf{Second:} Class 8 (\textsl{Ship}) is imbalanced in this setting and sees a sharp increase in its gradient norm and sensitivity. \textbf{Third:}There is a strong correlation between the gradient norm of groups and the hardware sensitivity in the ascending order of imbalance for UTKFace. \textbf{Fourth:} A similar trend is also found for the CelebA classification task.}
        \label{fig:gradient_norm_hardware_sensitivity}
\end{figure*}

\section{Experimental Results} 
\label{sec:experiments}

Our fairness analysis relies on the notion of hardware sensitivity, as defined in Equation (\ref{eq:2}), as the maximum difference in class loss between a model trained on a reference hardware and models trained on various other hardware setups, while keeping all other parameters unchanged to control software-related randomness. The notion of hardware sensitivity is of large theoretical value as it helps us understand the impact of hardware on model performance. However, ultimately, we are interested in measuring the accuracy variations across classes, due to training a model on different hardware. Therefore, in this section, we will also examine how hardware variations contribute to differences in accuracy across different groups.
When looking at hardware sensitivity, small values indicate more fair results.

In the models and architectures adopted, our analysis found notable fairness (hardware sensitivity) variations. Figure \ref{fig:hardware_sensitivity} illustrates this aspect for the CIFAR10 dataset. 

Firstly, observe that larger hardware sensitivity values for a class are associated with greater deviations in that class's accuracy. Next, also notice that classes showing smaller hardware sensitivity (indicative of greater fairness) tend to be those with higher overall accuracies.
To gain a better understanding of these trends, let us examine the hardware sensitivity of class 8 ({\sl Ship}) under both balanced and imbalanced scenarios, as depicted in Figure \ref{fig:hardware_sensitivity} (top). In the imbalanced setting, where class 8 (\textsl{Ship}) had five times fewer samples than other classes, there is a notable increase in hardware sensitivity from $0.017$ to $0.046$ compared to the balanced setting.

This pattern is not unique to one dataset. For instance, in the UTKFace dataset, the minority class {\sl Others} exhibits significantly higher hardware sensitivity of $0.122$ compared to $0.034$ of the majority class {\sl White}, this can be observed in \Cref{fig:dtd_utkface} (right). Similarly, in the CelebA dataset, the minority class {\sl blond-male} has a hardware sensitivity of $0.23$ which is more than $0.025$ of the majority class {\sl non-blond female} as can be seen in \Cref{fig:celeba_sensitivity} in the appendix. These observations support our hypothesis that hardware selection can disproportionately affect the performance of minority classes. 

Additionally, the paper provides insights on the relation between model architecture and its size and hardware sensitivity. These results are reported in \Cref{app:model_size} and provide strong indication that hardware sensitivity, and thus unfairness, increase as the complexity of the model increases.  

\begin{figure*}[t!]
    \centering
     \includegraphics[width=0.8\linewidth, trim={0 10pt 0 10pt}, clip]{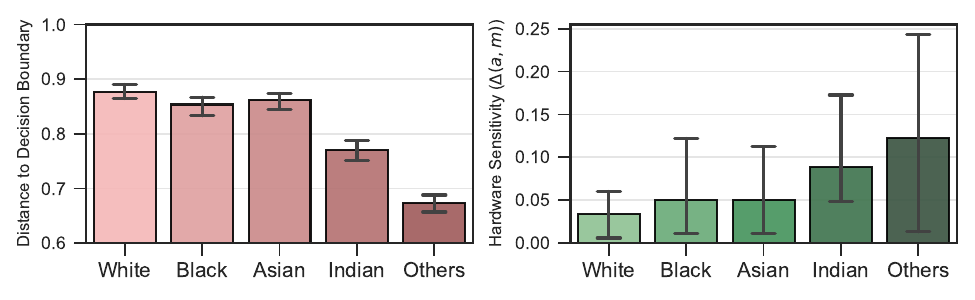}
    \caption{Relationship between distance to Decision Boundary and Hardware Sensitivity- UTKFace Ethnicity.}
    \label{fig:dtd_utkface}
\end{figure*}

\subsection{Gradient Flows} 
We now turn our attention to the influence of gradient flows on the disparities in accuracy resulting from hardware selection. As established in Theorem \ref{thm:taylor}, the magnitude of gradient flow within a group is directly linked to its hardware sensitivity.
The norm of the gradients, or the gradient flows, is indicative of the local optimality of the model for each group. Essentially, this term measures how sensitively the model responds to the specific characteristics within the data of each demographic group.

Larger gradient norm values suggest that the model is less optimized for that particular group, implying a greater potential for accuracy disparities due to hardware selection. 

Figure \ref{fig:gradient_norm_hardware_sensitivity} illustrates the relationship between the group gradient flows and the hardware sensitivity. Notice the strong correlation between a group's hardware sensitivity and its gradient flow, particularly under conditions of imbalance.
In CIFAR10, in particular, unbalancing class \textsl{Ship} (five times fewer samples) results in an increase in the gradient norm from $0.015$ to $0.041$. This trend is also echoed in the UTKFace-ethnicity task, where the gradient norm of the majority class \textsl{White} $0.05$ is significantly lower than $1.07$ of the minority class \textsl{Others}. CelebA shows a similar pattern; the majority class \textsl{non-blond-female} exhibits a gradient norm of $0.002$ which is much lesser than $0.25$ of the minority class \textsl{blond-male}.

These observations highlight the impact of class imbalance on gradient norms and hardware sensitivity, reinforcing the idea that minority classes tend to exhibit higher sensitivity to hardware variations, which in turn can affect the accuracy of the model for these specific groups.

\subsection{Distance to the Decision Boundary}
\label{sec:dist_bnd}
Next, we look at the second factor to unfairness highlighted in Theorem \ref{thm:taylor}: the effect of the maximum eigenvalue of the Hessian matrix of a group. Such values provide insight into the model's capacity to differentiate between the data of different groups. A larger maximum eigenvalue implies that the model's loss surface is more curved for the data of that particular group. This curvature is indicative of how sensitive the model is to variations in the data belonging to that group. Theorem \ref{thm:4} further links this component with the distance to the boundary, and we show next how such notion connects to hardware sensitivity (unfairness).

Figure \ref{fig:dtd_utkface} illustrates the relationship between the distance to decision boundary and hardware sensitivity for the UTKFace dataset. 
We adopt the definition of distance to the decision boundary from  \cite{TMF:neurips21}. For each sample $\bm{x}$, this distance is computed as $\delta_{\bm{x}} = 1 -\sum_{i=1}^{|\cY|} p^2_i(\bm{x})$, where \( p_i(\bm{x}) \) represents the softmax probability distribution of \(\bm{x}\) with values ranging between 0 and 1. Notably, the average distance to the decision boundary is a strong predictor of hardware sensitivity. In cases involving a minority class, such as \textsl{Others} in our dataset, this distance is significantly smaller ($0.67$) compared to other classes like \textsl{White} ($0.875)$.
We also note that other datasets follows similar trends: For example, on CelebA, the average distance to decision boundary is shorter ($0.56$) for \textsl{blond-male} compared to \textsl{non-blond-female} ($0.91$) as visible in \Cref{fig:dtd_celeba}. These findings aligns with the theoretical implications presented in our paper. 

\section{Mitigation Solution}
\label{sec:mitigation}

\begin{figure*}[t!]
    \centering
        \includegraphics[width=0.33\textwidth]{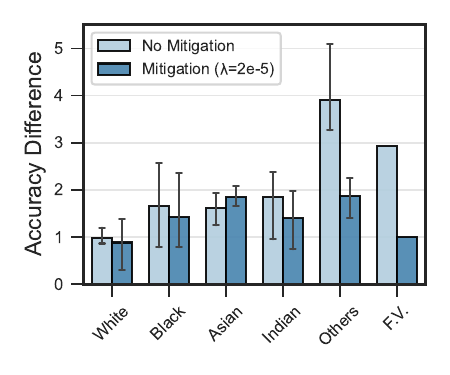}
    \hfill
        \includegraphics[width=0.33\textwidth]{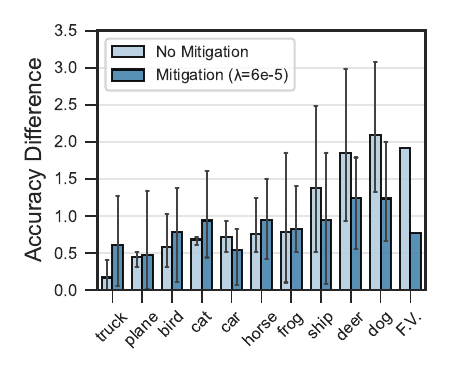}
    \hfill
        \includegraphics[width=0.33\textwidth]{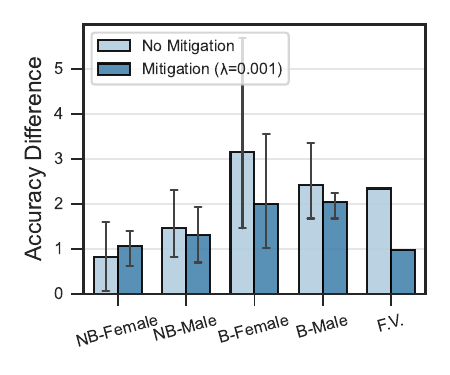}
    \caption{Accuracy difference across all hardware (analogous to hardware sensitivity) in percentage, for each class. 
    \textbf{Left}: Mitigation applied to UTKFace ($\lambda=2e-5$);
    \textbf{Middle}: Mitigation applied to CIFAR10 (Imbalanced) ($\lambda=6e-5$);
    \textbf{Right}: Mitigation applied to CelebA Dataset ($\lambda=0.001$).
    The last column in each subplot (named, F.V.) reports the fairness violation across all groups, which measures unfairness. 
    }
    \label{fig:mitigation}
\end{figure*}

Given the influence of the groups' gradient flows and group Hessians on the model unfairness due to hardware selection, one intuitive approach to mitigate the observed effects is to equalize the gradient and Hessian values across groups during training. However, this approach is computationally intensive and often impractical, especially for large models, primarily due to the challenges in computing the Hessian matrix during backpropagation. To address this issue, we propose a more efficient mitigation strategy, underpinned by the observations provided in  \Cref{thm:4}. This result elucidates the relationship between the group Hessian and the distance to the decision boundary. We leverage this insight and aim to align the average distance to the decision boundary among different groups.

We achieve this by augmenting the loss function with a component quantifying the disparity between the group-specific and batch-wide distances to the decision boundary.
\begin{align}
\label{eq:mitigation}
\optimal{\btheta}_{F} \,= \argmin_\btheta J(\btheta; D) 
+ \lambda \sum_{a \in \cA}(\delta_{\cD_a} - \delta_{\cD})^{2},
\end{align}
where $\delta_{S}$ represents the average distance to the decision boundary of samples $\bm{x} \in S$ as described in the \Cref{sec:dist_bnd}, and $\lambda$ is a hyper-parameter which calibrates the level of penalization associated with this term.

In our experiments, we implemented this mitigation strategy across various hardware setups and observed a significant reduction in fairness violation. 
While it is possible to optimize the choice of the value $\lambda$ during the empirical risk process, e.g., using a Lagrangian dual approach as in \cite{fioretto2020lagrangian,Fioretto:AAAI20,Tran:AAAI21}, we  found that even a traditional simple grid search allows us to  yield an effective reduction in accuracy disparity. 

\Cref{fig:mitigation} illustrates the effectiveness of the proposed method on UTKFace (left), CIFAR10 (middle), and CelebA (right). 
The plots report the accuracy difference, in percentage, across all hardware, for each class, as well as the models fairness violations (FV). The latter captures our unfairness metric, and is measured as the maximal accuracy difference across all hardware configurations, as detailed in \Cref{eq:3}. 

The results show a marked decrease in maximum difference in accuracy within various groups. For example, for the UTKFace dataset (left), despite experiencing a slightly increase in the accuracy difference for the {\sl Asian} group, we observe a three fold reduction in unfairness due to tooling. 
Similar effects are observed for the CIFAR10 dataset where, using $\lambda = 6e-5$, significantly reduces fairness violations, from 1.91 (pre-mitigation) to 0.77 (post-mitigation). 
Finally, for the ResNet50 model on CelebA (right), this mitigation reports a fairness violation reduction from 2.34 to 0.936 post-mitigation. It is to be noted that for each dataset, different $\lambda$ values produced different reductions in accuracy difference.

These results highlight the effectiveness of our proposed method in addressing fairness concerns attributable to hardware variations.

\section{Conclusion}
This paper focused on an often overlooked aspect of responsible model: How variations in hardware can disproportionately affect different demographic groups, leading to a Matthew's effect in performance and fairness. We've presented a theoretical framework to quantitatively assess these hardware-induced disparities, pinpointing variations in gradient flows across groups and differences in local loss surfaces as primary factors contributing to these disparities. These findings have been validated by extensive empirical studies, carried out on multiple hardware platforms, datasets, and architectures. 

\emph{The findings of this study are significant: the sensitivity of model performance to specific hardware choices can lead to unintended negative societal outcomes}. For example, organizations that release their source codes and model parameters may attest to satisfactory performance levels for certain demographic groups, based on results from their chosen training hardware. However, this claimed performance could substantially degrade when the models are implemented on different hardware platforms. Our work thus serves as both a cautionary tale and a guide for responsible practices in reporting across diverse hardware settings.

\subsection*{Impact Statement}
The analyses and solutions reported in this paper should not be intended as an endorsement for using the developed techniques to aid facial recognition systems. We hope this work creates further awareness of the unfairness caused by variations in data, model, and hardware setup.

\section*{Acknowledgments}
This research is partly funded by NSF grants SaTC-2133169, RI-2232054, and CAREER-2143706. F.~Fioretto is also supported by a Google Research Scholar Award and an Amazon Research Award. The views and conclusions of this work are those of the authors only. We would like to thank Cohere For AI for providing a generous amount of computing for conducting and analyzing our experiments. 

\bibliography{lib}
\bibliographystyle{icml2024}

\newpage
\appendix
\include{appendix}

\end{document}

%% file: math_commands.tex
\usepackage{amsmath,amsfonts,bm,bbm}
\usepackage{mathrsfs,amssymb,mathtools}
\usepackage{amsthm}

\def\eqref#1{equation~\ref{#1}}

\def\1{\bm{1}}

\DeclareMathAlphabet{\mathsfit}{\encodingdefault}{\sfdefault}{m}{sl}
\SetMathAlphabet{\mathsfit}{bold}{\encodingdefault}{\sfdefault}{bx}{n}

\DeclareMathOperator*{\argmin}{arg\,min}

\newcommand{\cA}{\mathcal{A}}  
 \newcommand{\cD}{\mathcal{D}}

\newcommand{\cM}{\mathcal{M}} 
\newcommand{\cO}{\mathcal{O}}

 \newcommand{\cY}{\mathcal{Y}}
 \newcommand{\cX}{\mathcal{X}}

 \newcommand{\RR}{\mathbb{R}}

\newcommand{\btheta}{{\bm{\theta}}}

\makeatletter
\newcommand{\oset}[3][0ex]{\mathrel{\mathop{#3}\limits^{\vbox to#1{\kern-2\ex@\hbox{$\scriptstyle#2$}\vss}}}}
\makeatother
\newcommand{\optimal}[1]{\oset{\scalebox{.5}{$\star$}}{#1}}

\usepackage{letltxmacro}
\LetLtxMacro\orgvdots\vdots
\LetLtxMacro\orgddots\ddots

%% file: appendix.tex
\onecolumn

\setcounter{theorem}{0}

\begin{center}
\noindent{\vspace{0.25em}\LARGE 
\textbf{\sc Supplemental Material}}
\end{center}

\section{Missing proofs} 
\label{app:proof}

\begin{theorem}
\label{thm:taylor_proof} 
Given a reference hardware $m$, the \emph{hardware sensitivity} of a group $a \in \cA$ is upper bounded by\footnote{
  With a slight abuse of notation, the results refer to $\bar{\btheta}$ as the homonymous vector which is extended with $k-\bar{k}$ zeros.
}: 
\begin{align}
  \Delta(a, m)  \leq 
  \left\| \bm{g}_{a,m}^{\ell} \right\|
  \times  \max_{m' \in \cM} \left\| \optimal{\btheta}_m - \optimal{\btheta}_{m'} \right\|
  + 
  \frac{1}{2} \lambda\left( \bm{H}_{a,m}^{\ell} \right) 
  \times 
  \max_{m' \in \cM} \left\| \optimal{\btheta}_m - \optimal{\btheta}_{m'}  \right\|^2 
  + 
  \cO\left( \max_{m' \in \cM} \left\| \optimal{\btheta}_m - \optimal{\btheta}_{m'}  \right\|^3 \right),
  \label{eq:thm1_appendix}
\end{align}
where 
$\bm{g}_{a,m}^\ell = \nabla_{\btheta} J( \optimal{\btheta}_m; D_{a})$ is the vector of gradients associated with the loss function $\ell$ evaluated at $\optimal{\btheta}_m$ and computed using group data $D_a$,  
$\bm{H}_{a,m}^{\ell} = 
\nabla^2_{\btheta} J(\optimal{\btheta}_m; D_{a})$ is the Hessian matrix of the loss function $\ell$, at the optimal parameters vector $\optimal{\btheta}_m$, computed using the group data $D_a$ (henceforth simply referred to as \emph{group hessian}), and  $\lambda(\Sigma)$ is the maximum eigenvalue of a matrix $\Sigma$.
\end{theorem}

\begin{proof}

Using a second order Taylor expansion around $\btheta^*_m$, the change in loss function of one particular group $a$ when it is trained on another hardware $m' \in \cM$  can be approximated as:
\begin{align}
  J(\btheta^*_{m'};D_a) - J(\btheta^*_m; D_a)
  &=    J\left(\btheta^*_m; D_{a}\right) 
    +  \left(\btheta^*_{m'} - \btheta^*_m \right)^\top \, 
    \nabla_{\theta} J\left(\btheta^*_m; D_{a} \right)  \notag\\
  & \hspace{15pt}+ \frac{1}{2} \left(\btheta^*_{m'}- \btheta^*_m\right)^\top\, 
  \bm{H}_{a}^\ell \left(\btheta*_{m'} - \btheta^*_m \right) 
   + \cO\left( \max_{m' \in \cM} \left\| \optimal{\btheta}_m - \optimal{\btheta}_{m'}  \right\|^3 \right) 
   - J\left(\btheta^*_m; D_{a}\right)  \notag\\
  &= \left(\btheta^*_{m'} - \btheta^*_m \right)^\top\,  \bm{g}^{\ell}_a 
  + \frac{1}{2} \left(\btheta^*_m - \btheta^*_m \right)^\top \,
  \bm{H}_a^\ell \left(\btheta^*_{m'} - \btheta^*_m \right) 
  +  \cO\left( \max_{m' \in \cM} \left\| \optimal{\btheta}_m - \optimal{\btheta}_{m'}  \right\|^3 \right)
  \label{eq:thm_1_1_appendix}
\end{align}
The above, follows from the loss $\ell(\cdot)$ being at least twice differentiable, by assumption.

By Cauchy-Schwarz inequality, it follows that:

\begin{equation}
    \left( \btheta^*_{m'} - \btheta^*_m \right)^\top 
    \bm{g}^{\ell}_a \leq 
    \left\| \btheta^*_{m'} - \btheta^*_m \right\| \times 
    \left\|\bm{g}^{\ell}_a \right\|.
    \label{eq:thm_1_2_appendix}
\end{equation}
In addition, due to the property of Rayleigh quotient we have:
\begin{equation}
    \frac{1}{2} \left(\btheta^*_{m'} - \btheta^*_m\right)^\top 
    \bm{H}_{a}^\ell \left(\btheta^*_{m'} - \btheta^*_{m} \right) 
    \leq 
    \frac{1}{2} \lambda \left(\bm{H}_a^\ell \right) \times 
    \left\|   \btheta^*_{m'} - \btheta^*_m \right\|^2.
    \label{eq:thm_1_3_appendix}
\end{equation}

Combine Equation \ref{eq:thm_1_1_appendix}, Equation \ref{eq:thm_1_2_appendix} and Equation  \ref{eq:thm_1_3_appendix} together we obtain the following upper bound:
\begin{align*}
  J(\btheta^*_{m'};D_a) - J(\btheta^*_m; D_a)
  \leq   \left\| \btheta^*_{m'} - \btheta^*_m \right\| \times 
    \left\|\bm{g}^{\ell}_a \right\| +  \frac{1}{2} \lambda \left(\bm{H}_a^\ell \right) \times 
    \left\|   \btheta^*_{m'} - \btheta^*_m \right\|^2.
\end{align*} 

By the definition of hardware sensitivity it follows that:

\begin{align*}
  \Delta(a,m) = \max_{m' \in \cM } | J(\btheta^*_{m'};D_a) - J(\btheta^*_m; D_a)| 
  \leq  \max_{m' \in \cM} \left\| \btheta^*_{m'} - \btheta^*_m \right\| \times 
    \left\|\bm{g}^{\ell}_a \right\| +  \frac{1}{2} \lambda \left(\bm{H}_a^\ell \right) \times 
     \max_{m' \in \cM} \ \left\|   \btheta^*_{m'} - \btheta^*_m \right\|^2.
\end{align*}

\end{proof}

\begin{theorem}
\label{thm:grad_imbalance} 
Consider a particular hardware $m \in \cM$, suppose for any group $a, a' \in \cA$ the angle between two gradient vectors $\overrightarrow{\boldsymbol{g}^{\ell}_{a,m}}; \overrightarrow{\boldsymbol{g}^{\ell}_{a',m}}$ is smaller than $\frac{\pi}{2}$. Then if we $ \underline{a} = \min_{a \in \cA} |D_a|$, i.e., the group with least number of training samples then the following holds:
\(
\| \bm{g}^{\ell}_{\underline{a},m}\| = \max_{a \in \cA} \| \bm{g}^{\ell}_{a,m}\|
\)
\end{theorem}

\begin{proof}

For notational convenience, denote $\bm{g}^{\ell}_m$ to be the gradient at convergence point over the whole dataset $D$. By the assumption, the gradient descent converges it follows that:

\begin{align}
\bm{g}^{\ell}_m = \sum_{a \in \cA} \frac{|D_a|}{|D|}\bm{g}^{\ell}_{a,m} = \bm{0}^T.
\end{align}
Consider the most minority group $\underline{a}$ (i.e, $|D_{\underline{a}}| = \argmin_{a \in \cA}  |D_a| $), it follows from the above equation that:

$$ \bm{g}^{\ell}_{\underline{a},m} = - \sum_{a \neq \underline{a}}\frac{|D_a|}{|D_{\underline{a}}|}\bm{g}^{\ell}_{a,m} $$

Taking the squared norm of vector on both sides of the previous equation, we have:

\begin{align}\| \bm{g}^{\ell}_{\underline{a},m} \|_2^2 = \left| \sum_{a \neq \underline{a}}\frac{|D_a|}{|D_{\underline{a}}|}\bm{g}^{\ell}_{a,m}\right|^2_2 = \sum_{a\neq \underline{a}} \|   \bm{g}^{\ell}_{a,m}\|^2_2 + 2 \sum_{a \neq a'\neq \underline{a}} ( \bm{g}^{\ell}_{a,m})^T\bm{g}^{\ell}_{a',m}
\end{align}

By the assumption that the angle between two gradient vectors of two arbitrary groups is less than $\frac{\pi}{2}$ hence $( \bm{g}^{\ell}_{a,m})^T\bm{g}^{\ell}_{a',m} \geq 0$. Thus it follows that:

\begin{align}\| \bm{g}^{\ell}_{\underline{a},m} \|_2^2 \geq \sum_{a\neq \underline{a}}  \|   \bm{g}^{\ell}_{a,m}\|^2_2 \geq \max_{a} \|    \bm{g}^{\ell}_{a,m}\|^2
\end{align}

Hence the smallest minority group will present the largest gradient norm. 

\end{proof}

\begin{theorem}
\label{thm:hessian_norm_bound} 
Let $f_{\btheta^*_m}$ be a binary classifier trained using a binary cross entropy loss on one reference hardware $m$. For any group $a \in \cA$, the maximum eigenvalue of the group Hessian $\lambda(\bm{H}_a^{\ell})$ can be upper bounded by:
\begin{equation}
    \lambda(\bm{H}_a^{\ell}) \leq \frac{1}{|D_a|}
    \sum_{(\bm{x}, y)\in D_a}
    \underbrace{
        \left( {f}_{\btheta^*_m}(\bm{x}) \right) 
        \left( 1 - {f}_{\btheta^*_m}(\bm{x}) \right)}_{\textit{Closeness to decision boundary}} 
        \times 
        \left\| \nabla_{\btheta} f_{\btheta^*_m}(\bm {x}) \right\|^2 
     + 
         \underbrace{\left| f_{\btheta^*_m}(\bm{x}) - y \right|}_{\textit{Error}} 
         \times 
        \lambda\left( \nabla^2_{\btheta} f_{\btheta^*_m}(\bm{x}) \right). 
    \label{eq:hessian_norm_bound}
\end{equation}
\end{theorem}

\begin{proof}
First notice that an upper bound for the Hessian loss computed on a group $a \in \cA$ can be derived as:
\begin{align}
    \lambda(\bm{H}_a^{\ell}) 
    &= \lambda\left( 
    \frac{1}{|D_a|} \sum_{(\bm{x}, y) \in D_a} \bm{H}_{\bm{x}}^{\ell} \right) 
    \leq 
    \frac{1}{|D_a|} \sum_{(\bm{x}, y) \in D_a}  
    \lambda\left( \bm{H}_{\bm{x}}^{\ell} \right)
  \label{eq:hessian_decompose}
\end{align}
where $\bm{H}_{\bm{x}}^{\ell}$ represents the Hessian loss associated with a sample  $\bm{x} \in D_a$ from group $a$.
The above follows Weily's inequality which states that for any two symmetric matrices $A$ and $B$,  $\lambda(A + B) \leq \lambda(A) + \lambda(B)$. 

Next, we will derive an upper bound on the Hessian loss associated to a sample $\bm{x}$. First, based on the chain rule a closed form expression for the Hessian loss associated to a sample $\bm{x}$ can be written as follows:
\begin{align} 
    \bm{H}^{\ell}_{\bm{x}} &=  
    \nabla^2_{f}   \ell\left(f_{\btheta^*_m}(\bm{x}), y\right)
    \left[ \nabla_{\btheta} f_{\btheta^*_m}(\bm{x}) 
    \left( \nabla_{\btheta} f_{\btheta^*_m}(\bm{x}) \right)^\top \right]
    + 
    \nabla_{f} \ell\left(f_{\btheta^*_m}(\bm{x}), y\right) 
    \nabla^2_{\btheta} f_{\btheta^*_m}(\bm{x}).
   \label{eq:hessian_sample}
\end{align}
The next follows from that 
\begin{align*}
\nabla_{f}   \ell\left( f_{\btheta^*_m}(\bm{x}), y\right) 
&= (f_{\btheta^*_m}(\bm{x}) - y), \\
\nabla^2_{f}   \ell\left(f_{\btheta^*_m}(\bm{x}), y\right) 
&= f_{\btheta^*_m}(\bm{x}) \left(1 - f_{\btheta^*_m}(\bm{x})\right).
\end{align*}
Applying the Weily inequality again on the R.H.S.~of Equation \ref{eq:hessian_sample}, we obtain:
\begin{align}
    \lambda(\bm{H}_{\bm{x}}^{\ell}) & \notag
    \leq  f_{\btheta^*_m}(\bm{x}) 
    \left(1 - f_{\btheta^*_m}(\bm{x})\right) \times 
    \left\| \nabla_{\btheta} f_{\btheta^*_m}(\bm{x}) \right\|^2 
    + \lambda\left( f_{\btheta^*_m}(\bm{x}) - y\right) \times 
    \nabla^2_{\btheta} f_{\btheta^*_m}(\bm{x})\\
    & \leq 
    f_{\btheta^*_m}(\bm{x}) 
    \left(1 - f_{\btheta^*_m}(\bm{x})\right) \times 
    \left\| \nabla_{\btheta} f_{\btheta^*_m}(\bm{x}) \right\|^2 
    + 
    \left| f_{\btheta^*_m}(\bm{x}) - y \right| 
    \lambda \left( \nabla^2_{\btheta} 
    f_{\btheta^*_m}(\bm{x}) \right)
    \label{eq:hessian_bound}
\end{align}

The statement of Theorem \ref{thm:hessian_norm_bound} is obtained combining Equations \ref{eq:hessian_bound} with \ref{eq:hessian_decompose}.
\end{proof}

\begin{proposition}
\label{appendix_dist_bndry}
 Consider a binary  classifier $f_{\btheta*_m}(\bm{x}) $ trained on one reference hardware $m$. For a given sample $\bm{x} \in D $, the term ${f}_{\btheta^*_m}(\bm{x}) (1 - {f}_{\btheta^*_m}(\bm{x}) )$ is maximized when ${f}_{\btheta^*_m}(\bm{x})  = 0.5$ and minimized when ${f}_{\btheta^*_m}(\bm{x}) \in \{0,1\} $.
\end{proposition}

\begin{proof}
First, notice that $f_{\btheta^*_m}(\bm{x}) \in [0,1]$, as it represents the soft prediction (that returned by the last layer of the network), thus ${f}_{\btheta^*_m}(\bm{x}) \geq f^2_{\btheta^*_m}(\bm{x})$. It follows that:

\begin{equation}
    f_{\btheta^*_m}(\bm{x}) \left( 1 -  f_{\btheta^*_m}(\bm{x}) \right) 
    =  f_{\btheta^*_m}(\bm{x}) - f^2_{\btheta^*_m}(\bm{x}) \geq 0.
\end{equation}
In the above, it is easy to observe that the equality holds when either $f_{\btheta^*_m}(\bm{x}) = 0$ or $f_{\btheta^*_m}(\bm{x}) = 1$.

Next, by the Jensen inequality, it follows that:
\begin{equation}
    f_{\btheta^*_m}(\bm{x}) \left( 1 -  f_{\btheta^*_m}(\bm{x}) \right) 
    \leq  \frac{\left( f_{\btheta^*_m}(\bm{x})+1 
                     - f_{\btheta^*_m}(\bm{x}) \right)^2}{4} 
    = \frac{1}{4}.
\end{equation}
The above holds when $f_{\btheta^*_m}(\bm{x}) = 1 - f_{\btheta^*_m}(\bm{x})$, in other words, when $f_{\btheta^*_m}(\bm{x}) = 0.5$. 
Notice that, in the case of binary classifier, this refers to the case when the sample $\bm{x}$ lies on the decision boundary. 
\end{proof}

\section{Choice of Hardware}\label{app:hardware}

We report experiments across widely adopted GPU types: Tesla T4 \citep{turing}, Tesla V100 \citep{volta}, Ampere A100 \cite{amperewhitepaper}  and Ada L4 GPU \citep{ada}. We choose this hardware because it represents a valuable variety of different design choices at a system level, and is also widely adopted across research and industry. Below we include additional context about how the design of this hardware differs.

\begin{table}[!thb]
\label{app:table_hardware}

    \captionsetup{justification=centering}
    \centering
    \begin{adjustbox}{max width=\textwidth}
    \begin{tabular}{lcccc}
    \toprule
    \textbf{Feature/Specification} & \textbf{Tesla V100} \citep{volta} & \textbf{Ampere A100} \cite{amperewhitepaper} & \textbf{Tesla T4} \citep{turing} & \textbf{Ada L4} \citep{ada} \\
    \midrule
    Hardware Architecture & Volta & Ampere & Turing & Ada Lovelace \\
    CUDA Cores & 5,120 & 6,912 & 2,560 & 7,424 \\
    Streaming Multiprocessors & 80 & 108 & 40 & 58 \\
    Total Threads & 163,840 & 221,184 & 81,920 & -- \\
    Tensor Cores & 640 & 432 (improved) & 320 & 232 (4th Gen) \\
    Memory & 16GB/32GB HBM2 & 40GB/80GB HBM2 & 16GB GDDR6 & 24GB GDDR6 w/ ECC \\
    Memory Bandwidth & Up to 900 GB/s & Up to 2,000 GB/s & Up to 320 GB/s & 300 GB/s \\
    FP32 Performance & 15.7 TFLOPS & 19.5 TFLOPS & 8.1 TFLOPS & 30.3 TFLOPS \\
    Interconnect & NVLink 2.0, 300 GB/s & NVLink 3.0, 600 GB/s & PCIe Gen 3 (32GB/s), No NVLink & -- \\
    TDP & 300W & 400W (variant-dependent) & 70W & 72 Watts \\
    \bottomrule
    \end{tabular}
    \end{adjustbox}
    \caption{Comparison of the feature design and system specifications of the hardware evaluated across all experiements. }
    \label{tab:gpu_comparison}
\end{table}

\textbf{NVIDIA V100 GPU} The NVIDIA V100 GPU, built upon the Volta microarchitecture \cite{voltawhitepaper}, introduced Tensor Cores as a notable innovation. Tensor Cores are specialized units that perform Fused-Multiply Add (FMA) operations, enabling the multiplication of two FP16 4x4 matrices with the addition of a third FP16 or FP32 matrix. 

\textbf{NVIDIA Tesla T4 GPU} The T4 is based upon the Turing Microarchitecture \cite{turingwhitepaper}, presented second-generation Tensor Cores capable of conducting FMA operations on INT8 and INT4 matrices. Despite both the T4 and the V100 sharing the same CUDA Core version and an equal number of CUDA cores per Streaming Multiprocessor (SM), the Tesla T4 GPU is specifically designed for inference workloads. Consequently, it incorporates only half the number of CUDA cores, SM Units, and Tensor Cores compared to the V100 GPU. Additionally, it utilizes slower GDDR6 memory, resulting in reduced memory bandwidth, but it is much more efficient in power consumption terms making it desirable for inference workloads.

The generous memory of the  V100 GPU relative to the T4 leads to increased speed of processing of the V100 GPU. This is because tensor cores are relatively fast, and typically the delay in processing is attributable to waiting for inputs from memory to arrive. With smaller memory, this means more retrieval trips.

\textbf{NVIDIA A100 GPU} The A100 GPU is based on the Ampere microarchitecture \cite{amperewhitepaper} and presents significant improvements relative to the V100 and T4. The A100 features faster up to 80 GB HBM2e memory, compared to the V100's upper limit of 20 GB HBM2 memory. The A100 provides more memory capacity and higher memory bandwidth, which allows for handling larger datasets and more complex models.

Although the number of Tensor cores per group was reduced from 8 to 4, compared to the Turing V100 GPU, these Third-generation Tensor Cores exhibit twice the speed of their predecessors and support newer data types, including FP64, TF32, and BF16. Furthermore, the Ampere architecture increased the number of CUDA cores and SM Units to 6,912 and 108, respectively, resulting in a notable 35\% increase in the number of threads, compared to the V100 GPU, that can be processed in parallel.

\textbf{NVIDIA L4 GPU} The Ada Lovelace Microarchitecture \cite{ada}, was designed on TSMC's 5nm Process Node, leading to an increased Performance Per Watt Metric. This allowed Nvidia to pack in more CUDA Cores within a single Streaming Multiprocessor (SM), leading to an increased FLOPS Throughput. The L4 GPU, of the Ada Lovelace Microarchitecture, is an inference-friendly GPU, Leading to TDP being fixed at 72 Watts, allowing High Energy Efficiency. To reduce cost, Nvidia opted for the GDDR6 Memory instead of the High Performance HBM, found in their flagship GPUs. The Fourth-generation Tensor Cores support new Datatypes like FP8 (With Sparsity), allowing a much higher throughput. Also, the number of Tensor Cores per SM has been increased. L4 is a substantial improvement from the previous inference-friendly GPU T4.

\subsection{Datasets}
\label{app:datasets}

\textbf{CIFAR10 and CIFAR100} \cite{cifar10} are datasets which contain colored natural images of size 32 x 32. In CIFAR10, there are 10 classes of objects with a total of 60000 images (50000 train - 5000 per class, 10000 test -- 1000 per class).

\textbf{Imbalanced versions} For our experiments we benchmark two versions of the CIFAR10 dataset, a \emph{Balanced} version which is the original dataset described above and an \emph{Imbalanced} version. The 'Imbalanced' version is a modified version of the original where the class 8 (Ship - CIFAR10) has been reduced to 20\% of their original size. The other classes are not modified.

\textbf{UTKFace} The UTKFace \cite{utkface} is a large scale dataset of face images. This dataset has 20,000 images with annotations for age, gender and ethnicity and images taken in a variety of conditions and image resolutions. It is naturally imbalanced with respect to ethnicity, which provides a challenging and informative setting for our experiments. In this paper, we investigated classification using the ethnicity annotation. The task we perform is image classification -- there are 5 class labels: Asian, Indian, White, Black and Others. This is a useful task as it allows us to investigate the disparate effect of tooling on a task where the dataset is naturally imbalanced and highlights a sensitive use case involving protected attributes.

\textbf{CelebA} The CelebA \cite{celeba} is an image dataset that consists of around 202,599 face images with 40 associated attribute annotations. For this task, we aim to classify face images into 4 distinct classes: 'Blond Male', 'Blond Female', 'Non-Blond Male' and 'Non-Blond Female.' This is also a naturally imbalanced task, with `Blond Male' being the minority class. Here, gender is a protected attribute and our goal is to understand how hardware amplifies the bias.  

\subsection{Architectures}
\label{app:architectures}

\textbf{SmallCNN} - We use a custom Convolutional Neural Network with 5 convolutional layers, 3 linear layers and one MaxPooling layer with stride = 2. Using SmallCNN as the base architecture enabled us to explore an extensive ablation grid while making effective use of computational resources.

\textbf{ResNet18, ResNet34 and ResNet50} \cite{resnet} - includes residual blocks and has become an architecture of choice for developing computer vision applications. We evaluate two variants, namely ResNet18 and ResNet34 and ResNet50 with 18, 34 and 50 layers respectively. The versions used in this code were the default implementations available in the torchvision library \citet{torchvision2016}.

\textbf{Controlling model stochasticity.} Stochasticity is typically introduced into deep neural network optimization by factors including algorithmic choices, hardware and software \citep{zhuang2022randomness}. Our goal is to precisely measure the impact of tooling on the fairness and performance of the model. Hence, we seek to control stochasticity introduced by algorithmic factors, to disambiguate the impact of noise introduced by hardware.

The stochasticity arising from algorithmic factors was controlled as follows: the experimental setup maintained a fixed random seed across all Python libraries, including PyTorch \cite{pytorch} 2.0, ensuring consistency. We ensure that the data loading order and augmentation properties were controlled using a fixed seed through FFCV-SSL \citep{FFCV_SSL}, a fork of FFCV \cite{ffcv}. 

A critical part of our analysis requires that there is a fair comparison between different hardware platforms used for both training and inference. To ensure consistent experimental configuration across  hardware platforms, we fix the parameters related to the training harness for a given dataset and model. It includes but is not limited to batch size, learning rate, initialization, and optimizer. The models trained on UTKFace and CIFAR-10 for both settings were in full-precision (FP32) for both training and inference.  Models trained on the CelebA dataset, we employed mixed-precision training due to memory and time constraints. For these experiments, we use float16 as the intermediate data type. The inference, however, takes place in full precision (FP32).  Reported metrics were averaged across runs gathered from approximately five random seeds.

While the theoretical analysis focuses on the notion of disparate impacts under the lens of hardware sensitivity with respect to the risk functions, the empirical results which we report are differences in the accuracy of the resulting models across different hardware. This way the empirical results thus reflect the setting commonly adopted when measuring accuracy parity \citet{zhao2019inherent} across groups. In addition, we also report metrics on gradient norm, Hessian's max eigenvalue, and the average distance from the decision boundary for various groups in the datasets which highlights optimization differences amplified by tooling, which could lead to an increase in hardware sensitivity and shown in the paper, unfairness.

\section{Model Convergence}
\label{app:convergence}
We report training loss curves and test accuracy for each dataset adopted. The results are shown in Figure \ref{fig:metrics}. 
Notice that, in our experiments, in the paper, we identified a stable training configuration for each dataset, where further training led to overfitting. %
For comparison, here, we report results on a ResNet34 model trained for 100 epochs. The learning rate scheduler, \texttt{OneCycleLR}, is dependent on the total number of epochs specified initially, resulting in faster performance improvement with fewer epochs, which is the reason for the different slopes in the curves. 
Notice, that beyond the purple line that we call the "significant epoch" there is diminishing returns in terms of improvement in accuracy, and for datasets like CelebA and UTKFace, over-fitting is observed. 

Recall that our analysis focused on hardware sensitivity to measure unfairness. As shown in Figure \ref{fig:hardware_sensitivity_cifar}, our analysis remains unaltered, regardless of the number of training epoch adopted. While for larger number of epochs the sensitivity values decrease slightly in some datasets, the overall trends and the phenomenon is analogous to what we observed in the paper, even with improved test accuracy. This supports our argument presented in the rebuttal.

As a secondary point, it is also important to note also that our training was limited by processing costs and capabilities of the slowest GPU, making computational efficiency a crucial factor due to the extensive nature of our experiments.

\begin{figure}[ht!]
    \centering
    \includegraphics[width=0.95\textwidth]{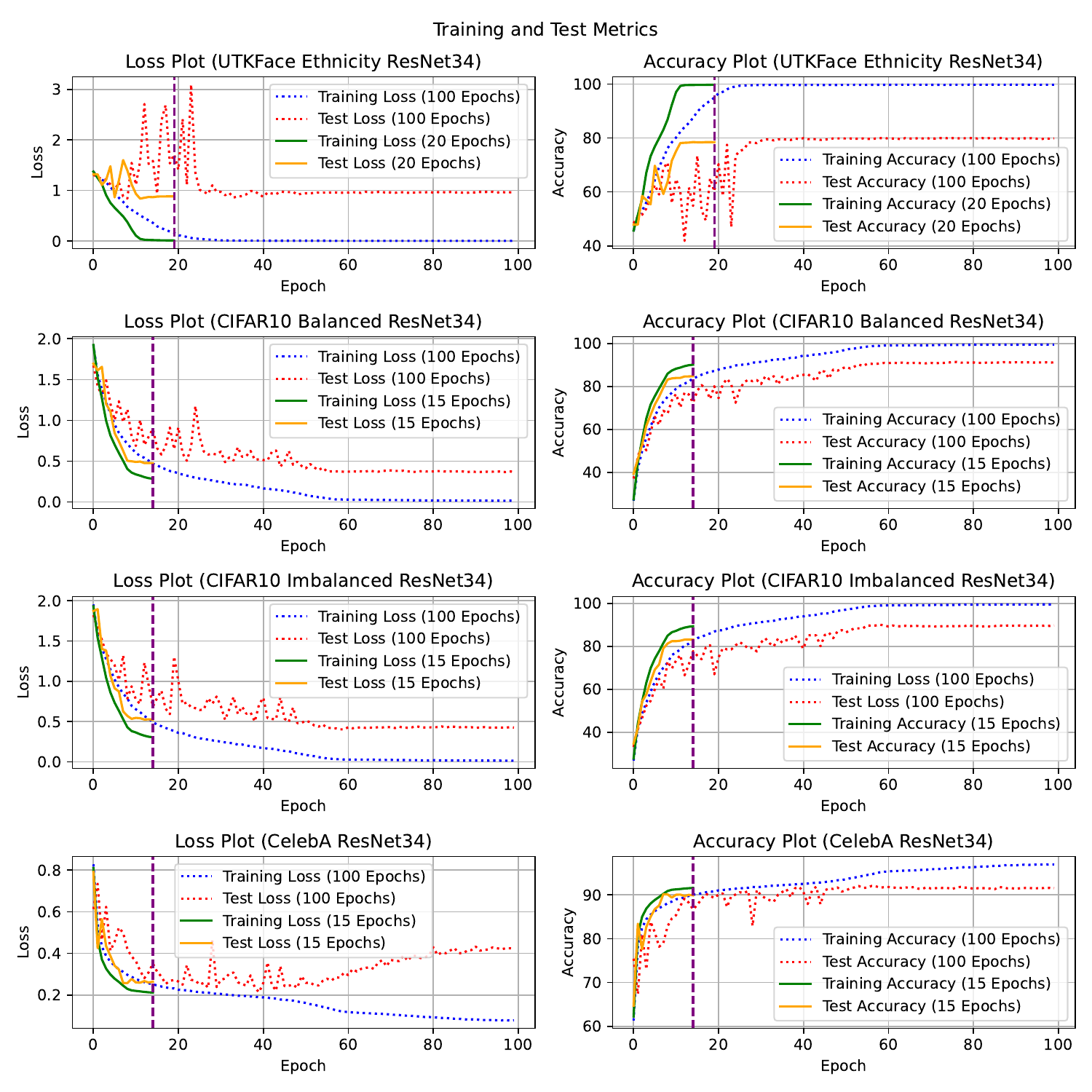}
        \caption{Training and test loss (left) and training and test accuracy (right) for models trained with 100, 20, or 15 epochs. 
        First row: UTKFace; 
        Second row: CIFAR10 balanced;
        Third row: CIFAR10 unbalanced;
        Fourth row: CelebA;
        In each plot, the purple dotted vertical line indicates the training epoch used in the main paper for that specific dataset. Beyond this epoch, continued training leads to no significant impact to hardware sensitivity.}
    \label{fig:metrics}
\end{figure}

\begin{figure}[t!]
    \centering
    \includegraphics[width=0.45\textwidth]{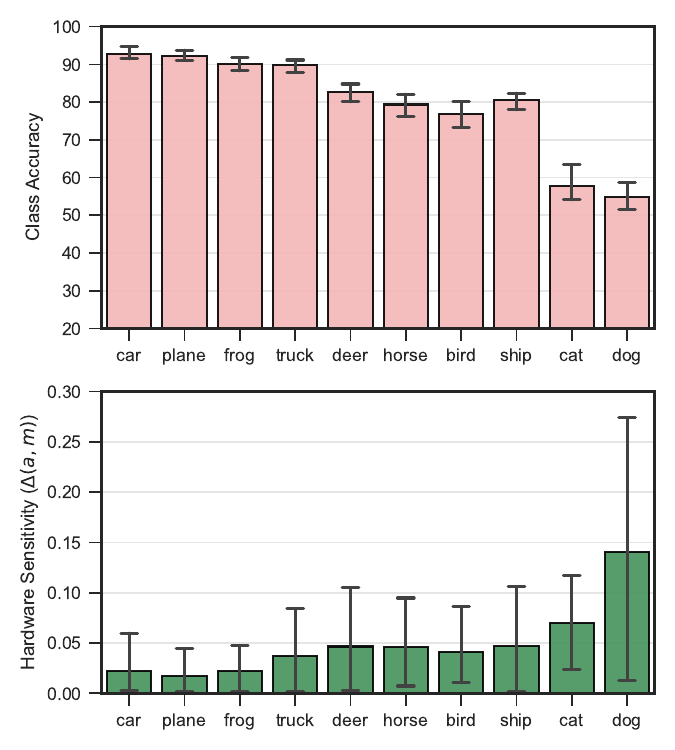}
        \includegraphics[width=0.45\textwidth]{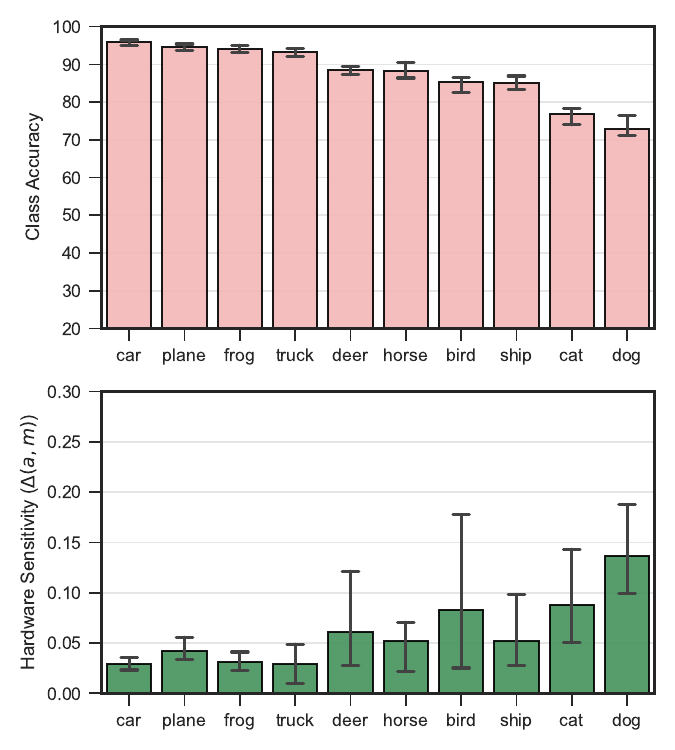}
        \caption{
        Class accuracy (top) and Hardware sensitivity (bottom) --- our fairness metric --- for CIFAR-10 (unbalanced). The left plot report models trained with 15 epochs and the right plots models trained over 100 epochs. 
        Notice that, while some accuracy difference is observable for cat and dog classes, the hardware sensitivity trends are not affected.}
    \label{fig:hardware_sensitivity_cifar}
    \end{figure}

\begin{figure}[t!]
    \centering
    \begin{subfigure}[b]{0.45\textwidth}
        \centering
        \includegraphics[width=\textwidth,trim={0 14pt 0 0}, clip]{imgs/mitigation_eps_fairness_sensitivity_utkface_2e-5.pdf}
        \caption{UTKFace: $\lambda=2e-5$}
    \end{subfigure}
    \begin{subfigure}[b]{0.45\textwidth}
        \centering
        \includegraphics[width=\textwidth,trim={0 14pt 0 0}, clip]{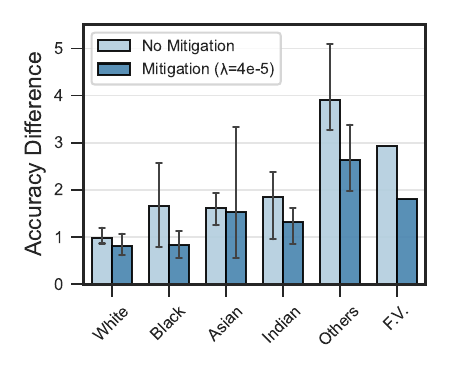}
        \caption{UTKFace: $\lambda=4e-5$}
    \end{subfigure}
    
    \begin{subfigure}[b]{0.45\textwidth}
        \centering
        \includegraphics[width=\textwidth,trim={0 14pt 0 0}, clip]{imgs/mitigation_eps_fairness_sensitivity_cifar_6e-5.pdf}
        \caption{CIFAR-10: $\lambda=6e-5$}
    \end{subfigure}
    \begin{subfigure}[b]{0.45\textwidth}
        \centering
        \includegraphics[width=\textwidth,trim={0 14pt 0 0}, clip]{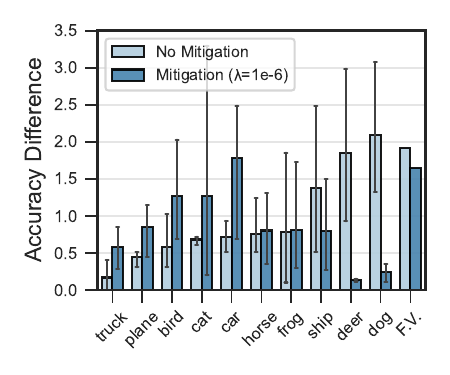}
        \caption{CIFAR-10: $\lambda=1e-6$}
    \end{subfigure}
    
    \caption{
    \textbf{UTKFace and CIFAR-10:} 
    Mitigation solution using different values of $\lambda$. 
    The y-axis describes the accuracy difference across all hardware (analogous to hardware sensitivity) for each class (x-axis). The last column of the x-axis (named, F.V) reports the fairness violation across all groups, which measures unfairness.
    }
    \label{fig:combined_miti}
\end{figure}

\section{Fairness/Accuracy Tradeoffs}
\label{app:tradeoff}
We report additional results on the fairness-accuracy trade-offs resulting from the proposed mitigation mechanism. 
Recall that our mitigation mechanism augments the loss function with a component quantifying the disparity between the group-specific and batch-wide distances to the decision boundary:
\begin{align}
\label{app:mitigation}
\optimal{\btheta}_{F} \,= \argmin_\btheta J(\btheta; D) 
+ \lambda \sum_{a \in \cA}(\delta_{\cD_a} - \delta_{\cD})^{2},
\end{align}
where $\delta_{S}$ represents the average distance to the decision boundary of samples $\bm{x} \in S$, and $\lambda$ is a hyper-parameter which calibrates the level of penalization.

In the following experiment, we report results on the largest accuracy difference across hardware at the varying of the hyper-parameter $\lambda$ chosen within values $\{0.01, 0.001, 0.0001, 1e-5, 2e-5, 4e-5, 6e-5, 8e-5, 1e-6\}$ for CIFAR10 and UTKFace. 

The results are illustrated in \Cref{fig:combined_miti} for CIFAR-10 and UTKFace, where we only report a few parameters representative of the major change among the various class accuracy difference across various hardware. 

We observed that setting the parameter $\lambda = 4e-5$ results in a significant reduction in unfairness when compared to the unmitigated solution. Unfairness here is measured by the maximal accuracy difference across all hardware configurations (which is the \textbf{fairness violation} referred to as \textbf{F.V.} in the figure and displayed as the last column) detailed in equation 3 of the main paper. However, a smaller value of $\lambda = 2e-5$ leads to an even greater reduction in unfairness (F.V.), despite increasing the accuracy difference (hardware sensitivity) for the "Asian" group.

Similar effects are observed for the CIFAR dataset. Using $\lambda = 6e-5$ significantly reduces fairness violation (F.V.), decreasing it from 1.91 (unmitigated solution) to 0.77 (post-mitigation). However, this setting tends to increase hardware sensitivity for certain classes (e.g., bird, cat, and horse). On the other hand, a smaller value of $\lambda = 1e-6$ results in the most substantial decrease in hardware sensitivity for classes such as deer and dog. Nevertheless, this value does not reduce unfairness (F.V.) as effectively as $\lambda = 6e-5$.

Therefore, as stated in the paper, while it is possible to optimize the choice of the value $\lambda$ during the empirical risk process, e.g., using a Lagrangian dual approach as in \citep{fioretto2020lagrangian,Fioretto:AAAI20}, we found that even a traditional simple grid search allows us to find good $\lambda$ values yielding an effective reduction in accuracy disparity.

\section{Model Architecture/Size}
\label{app:model_size}
We report a summary of our observations related to model capacity and its relation to hardware sensitivity.
Figure \ref{fig:model_hardware_sensitivity} reports the class accuracy (first and third rows) and Hardware sensitivity (second and fourth rows) for UTKFace Ethnicity (top rows) and CIFAR10 (Imbalanced) (bottom rows). \textbf{Unfairness} is reported in the last histogram of the green plots (second and fourth rows), and indicates the maximal pairwise class difference in accuracy across hardware, consistently with the notion adopted in the paper and above (named F.V) The three columns display results obtained using three different architectures of increasing complexity: A small CNN (527,754 parameters) , a ResNet18 (11.4M parameters) and a ResNet34 (21.7M parameters). 

Firstly, notice that unfairness due to hardware tooling is widespread across all datasets and architectures tested. Next, notice that the smallest architecture (left) exhibits high hardware sensitivity for certain classes, such as ship, dog, and cat in CIFAR10 and others in UTKFace. This is likely due to under-fitting. The ResNet18 (middle) has a moderate impact on reducing hardware sensitivity. However, transitioning to an ever deeper model, ResNet34 (right) results in significant increase in fairness violations --- correspondingly hardware sensitivity increases sharply for minority classes.  This exacerbates the overall fairness violation of the model.  Specifically, the fairness violation  increases by approximately 68\% for ResNet34 (0.27) compared to ResNet18 (0.16) and SmallCNN (0.17) in CIFAR10, and nearly doubles for ResNet34 (0.294) versus ResNet18 (0.14) and SmallCNN (0.12) in UTKFace Ethnicity.

\clearpage
\section{Additional plots for CelebA}

In this section, we report figures related to our experiments on the CelebA dataset. We find the trends to be similar to the other datasets, and support our observations further. Below we provide figures related to hardware sensitivity, distance to decision boundary and mitigation experiments on CelebA.

\begin{figure}[t!]
    \centering
    \includegraphics[width=0.95\linewidth]{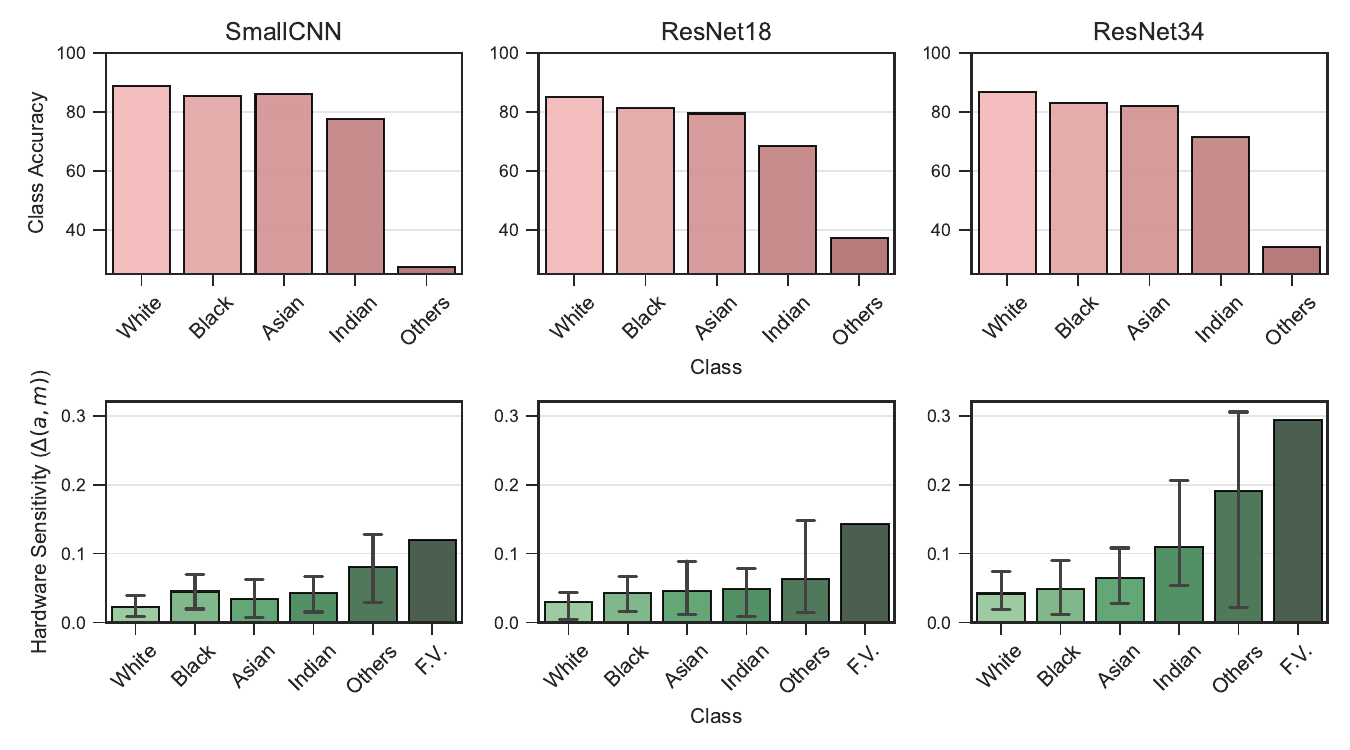}
    \includegraphics[width=0.95\linewidth]{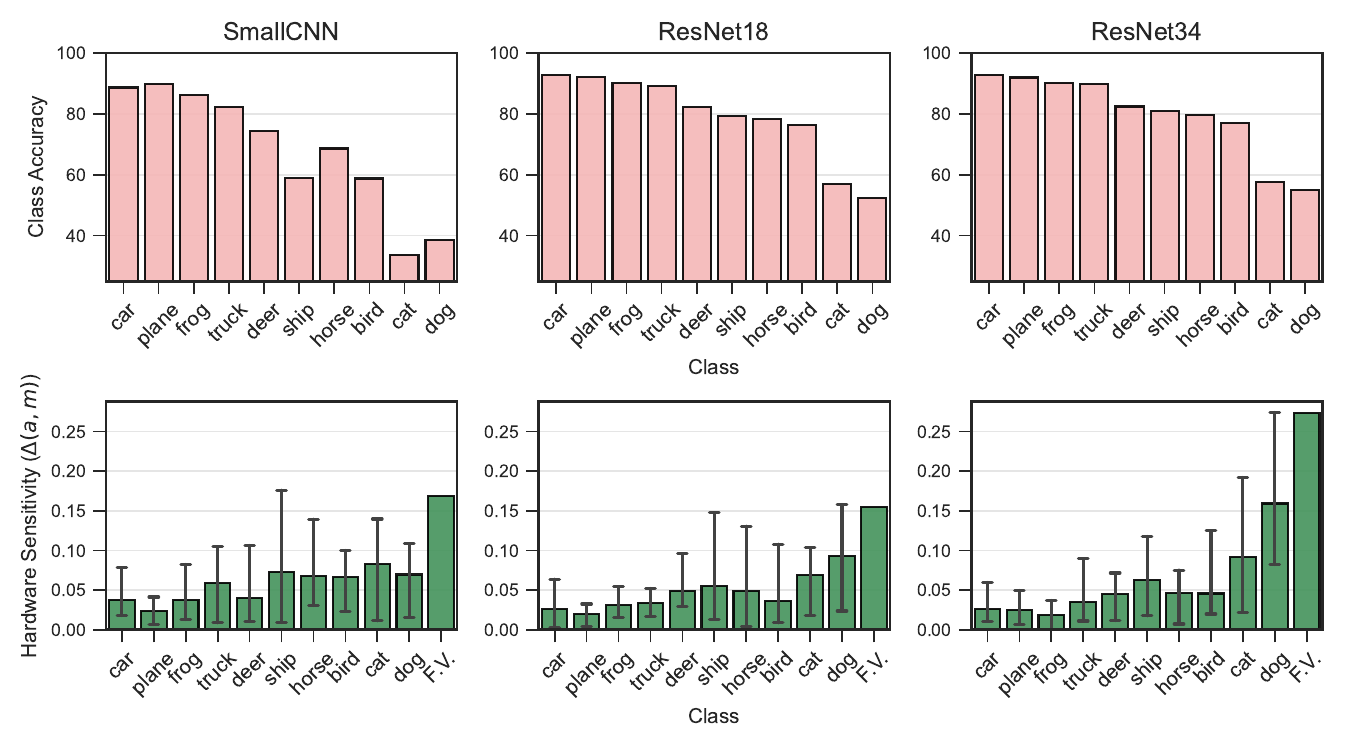}
        \caption{\textbf{Left to right:} Hardware Sensitivity Across Architectures. Notice as we increase the model size, the Fairness impact due to hardware choice increases. \textbf{Top:} Class-wise accuracy and Hardware Sensitivity for UTKFace Ethnicity. \textbf{Bottom: } For CIFAR-10.}
        \label{fig:model_hardware_sensitivity}
\end{figure}

\begin{figure*}[t!]
    \centering
        \includegraphics[width=0.4\textwidth]{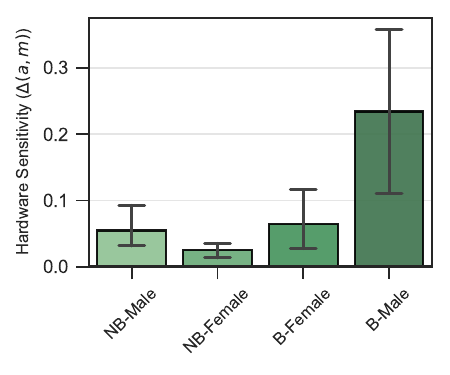}
        \includegraphics[width=0.4\textwidth]{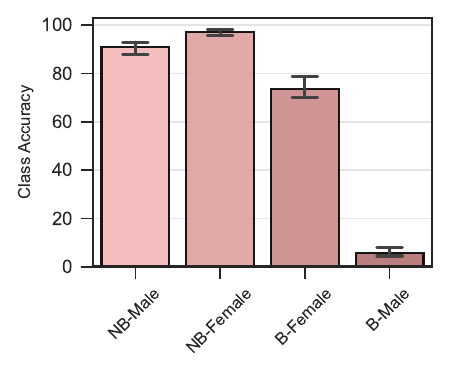}
        \caption{\textbf{Left:} Hardware Sensitivity for CelebA (ResNet34). \textbf{Right:} Class-wise accuracy}
        \label{fig:celeba_sensitivity}
\end{figure*}

\begin{figure*}
    \centering
        \includegraphics[width=0.85\textwidth]{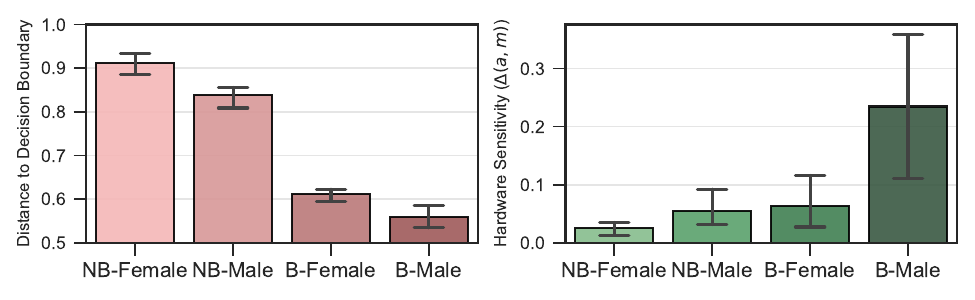}
        \caption{\textbf{Left: } Distance-to-decision Boundary for CelebA on ResNet34 \textbf{Right: } Hardware Sensitivity }
    \label{fig:dtd_celeba}
\end{figure*}

\begin{figure*}[t!]
     \centering
     \begin{subfigure}[a]{\textwidth}
         \centering
         \includegraphics[width=0.85\textwidth]{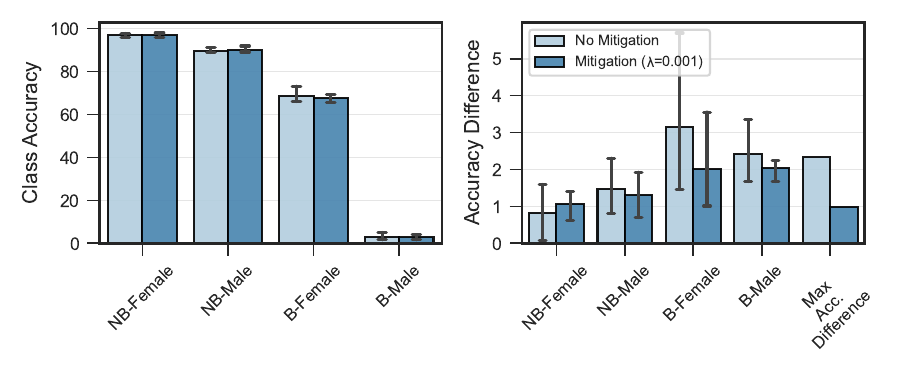}
     \end{subfigure}
        \caption{\textbf{Left:} Class-wise Accuracy for groups pre and post mitigation. \textbf{Right :} Accuracy Difference for groups pre and post mitigation. Notice the Maximum Accuracy Difference between the maximum and minimum accuracy within groups is generally reduced post-mitigation averaged across hardware. CelebA on ResNet50.}
        \label{fig:fairnessmitigation}
\end{figure*}